\documentclass{article}
\usepackage[T1]{fontenc}
\usepackage{microtype}
\usepackage{graphicx}
\usepackage{subcaption}
\usepackage{booktabs} 

\usepackage{hyperref}
\usepackage{multirow}
\usepackage[table]{xcolor} 

\usepackage[preprint]{icml2026}

\usepackage{amsmath}
\usepackage{amssymb}
\usepackage{mathtools}
\usepackage{amsthm}
\usepackage{algorithm}
\usepackage{algorithmic}

\usepackage[capitalize,noabbrev]{cleveref}

\theoremstyle{plain}
\newtheorem{theorem}{Theorem}[section]
\newtheorem{proposition}[theorem]{Proposition}

\theoremstyle{definition}
\newtheorem{definition}[theorem]{Definition}
\newtheorem{assumption}[theorem]{Assumption}
\theoremstyle{remark}

\newcommand{\framework}{RoboStriker}
\usepackage[textsize=tiny]{todonotes}

\icmltitlerunning{RoboStriker: Hierarchical Decision-Making for Autonomous Humanoid Boxing}
\begin{document}

\twocolumn[
  \icmltitle{RoboStriker: Hierarchical Decision-Making for Autonomous Humanoid Boxing}



  \icmlsetsymbol{equal}{*}

  \begin{icmlauthorlist}
    \icmlauthor{Kangning Yin}{equal,aaa,bbb,ccc}
    \icmlauthor{Zhe Cao}{equal,aaa,bbb}
    \icmlauthor{Wentao Dong}{aaa,bbb}
    \icmlauthor{Weishuai Zeng}{bbb,ddd}
    \icmlauthor{Tianyi Zhang}{aaa}
    \icmlauthor{Qiang Zhang}{eee}
    \icmlauthor{Jingbo Wang}{bbb}
    \icmlauthor{Jiangmiao Pang}{bbb}
    \icmlauthor{Ming Zhou}{bbb}
    \icmlauthor{Weinan Zhang}{aaa,bbb,ccc}
  \end{icmlauthorlist}

  \icmlaffiliation{aaa}{Shanghai Jiao Tong University}
  \icmlaffiliation{bbb}{Shanghai Artificial Intelligence Laboratory}
  \icmlaffiliation{ccc}{Shanghai Innovation Institute}
  \icmlaffiliation{ddd}{Peking University}
  \icmlaffiliation{eee}{The Hong Kong University of Science and Technology (Guangzhou),}

  \icmlcorrespondingauthor{Ming Zhou}{zhouming@pjlab.org.cn}
  \icmlcorrespondingauthor{Weinan Zhang}{wnzhang@sjtu.edu.cn}

  \icmlkeywords{Machine Learning, ICML}

  \vskip 0.3in
]



\printAffiliationsAndNotice{}  

\begin{abstract}
  Achieving human-level competitive intelligence and physical agility in humanoid robots remains a major challenge, particularly in contact-rich and highly dynamic tasks such as boxing. While Multi-Agent Reinforcement Learning (MARL) offers a principled framework for strategic interaction, its direct application to humanoid control is hindered by high-dimensional contact dynamics and the absence of strong physical motion priors. We propose \framework{}, a hierarchical three-stage framework that enables fully autonomous humanoid boxing by decoupling high-level strategic reasoning from low-level physical execution. The framework first learns a comprehensive repertoire of boxing skills by training a single-agent motion tracker on human motion capture data. These skills are subsequently distilled into a structured latent manifold, regularized by projecting the Gaussian-parameterized distribution onto a unit hypersphere. This topological constraint effectively confines exploration to the subspace of physically plausible motions. In the final stage, we introduce Latent-Space Neural Fictitious Self-Play (LS-NFSP), where competing agents learn competitive tactics by interacting within the latent action space rather than the raw motor space, significantly stabilizing multi-agent training. Experimental results demonstrate that \framework{} achieves superior competitive performance in simulation and exhibits sim-to-real transfer. Our website is available at \href{https://yinkangning0124.github.io/RoboStriker/}{RoboStriker}.
\end{abstract}

\begin{figure}
    \vspace{-2mm}
    \includegraphics[width=\columnwidth]{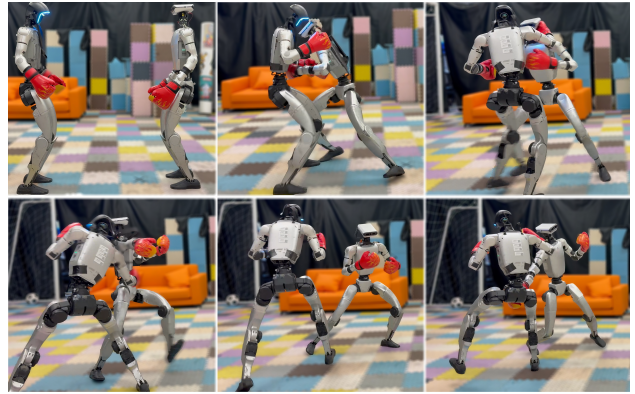}
    \caption{Real-world clips of humanoid boxing using RoboStriker, showcasing agile, contact-rich punches and defenses under physical constraints.}
    \vspace{-2mm}
    \label{fig:realworld}
\end{figure}

\section{Introduction}

Humanoid robots serve as a critical frontier for embodied intelligence, with recent breakthroughs demonstrating exceptional capabilities in locomotion and agile maneuvers~\cite{liao2025beyondmimicmotiontrackingversatile, yin2025unitracker, zeng2025behaviorfoundationmodelhumanoid, zhang2025track}. These successes, however, have largely focused on single-agent execution, where the environment is passive and predefined. Once other agents are introduced, as in competitive physical interactions such as boxing, the problem fundamentally shifts from motion-generation to strategic co-adaptation under physical constraints.
This shift naturally situates humanoid competition within the scope of Multi-Agent Reinforcement Learning (MARL).
Yet, unlike symbolic or abstract domains where MARL has achieved grandmaster-level success~\cite{silver2017mastering, vinyals2019grandmaster}, embodied humanoid competition requires reasoning through high-dimensional, contact-rich dynamics, where every strategic decision is physically bound by balance stability and actuation limits. At this frontier of multi-agent embodiment, we argue that there lie two intrinsic and domain-agnostic contradictions that characterize the competitive humanoid systems.

\vspace{-3mm}
\paragraph{Physical Feasibility vs. Non-stationary Learning.}
Humanoid control operates in a high-dimensional continuous action space, which is subject to strict physical constraints, including contact stability, joint limits and underactuated dynamics. Effective strategy exploration, however, requires diverse and sometimes aggressive behaviors. As a result, such unconstrained policy exploration often leads to physically invalid motions and fails to keep balance~\cite{peng2018deepmimic, peng2021amp}, while overly restrictive controllers collapse strategy learning. This further creates an embodied cold-start problem: agents cannot evolve competitive strategies without first maintaining a stable physical stance~\cite{luo2023universal}.
\vspace{-3mm}
\paragraph{Strategy Evolution vs. System Stability.}
MARL via self-play~\cite{samuel1959some,hernandez2019generalized} relies on continuously evolving opponent policies, which introduces severe non-stationarity during training. While such dynamics are essential for discovering competitive strategies, humanoid systems are dynamically fragile: small distributional shifts in opponent behavior can destabilize balance and contact, leading to training collapse. This results in another tension between competitive strategy evolution and stable embodied interaction.

Existing approaches address these challenges in an isolated manner. Game-theoretic MARL methods such as Neural Fictitious Self-Play (NFSP)~\cite{heinrich2016deep} are designed for abstract or weakly grounded environments, lacking inductive biases for physical feasibility. Conversely, embodied control frameworks such as Adversarial Motion Priors (AMP)~\cite{peng2021amp} and DeepMimic~\cite{peng2018deepmimic} excel at learning robust single-agent motor skills via imitation, but do not support strategic co-evolution or opponent-aware adaptation. As a result, neither line of work alone can resolve the dual contradictions inherent to embodied competitive tasks.

In this work, we argue that resolving these challenges requires a structured decoupling of physical control and strategic reasoning, coupled with a training mechanism that stabilizes competitive evolution. To this end, we propose a hierarchical framework that decomposes embodied MARL into three coupled layers: (1) physically grounded motion library, (2) structured latent motion space for strategy representation and (3) multi-agent strategy evolution over this latent space.
Concretely, in the first layer, we establish a DeepMimic-based tracking policy~\cite{peng2018deepmimic} to faithfully reproduce diverse human motion primitives, ensuring the agent masters a repertoire of physically feasible kinematic skills. In the second layer, we construct a bounded latent motion space that supports diverse motion generation and serves as the effective strategy space. In the third layer, we perform competitive strategy learning via NFSP, namely Latent Space NFSP (LS-NFSP), which enables latent multi-agent co-evolution within the bounded latent manifold.
Crucially, the bounded latent motion space is not merely a representational choice, it defines a compact, physically safe strategy search space that allows LS-NFSP to operate without inducing motor-level instability. Meanwhile, AMP-based~\cite{peng2021amp} curriculum warmup initializes the policy within this manifold, mitigating the competitive cold-start problem and substantially reducing the non-stationarity during fictitious self-play. Together, these mechanisms resolve the aforementioned contradictions of embodied MARL by reconciling physical stability with strategic diversity, and stabilizing strategy evolution under continuous competition.

We evaluate our framework on a humanoid boxing task using Unitree G1 robots with 29 degrees of freedom~\cite{unitreeg1}, demonstrating that the proposed method achieves substantially improved performance, robustness, and convergence stability compared to existing baselines. Beyond this specific task, our framework provides a general recipe for embodied multi-agent competition, offering a principled pathway to scale MARL from abstract games to physically grounded robotic systems. Our contributions are summarized as follows:
\begin{itemize}
    \item To the best of our knowledge, this work is the first to formally characterize the intrinsic contradictions in embodied MARL, specifically the trade-off between physical feasibility and non-stationary learning, and the conflict between strategy evolution and system stability, using humanoid competition as a prime instantiation.
    \item We propose a hierarchical framework that decouples high-level strategic reasoning from low-level physical execution, providing a stable pathway for evolving dynamic combat behaviors.
    \item We demonstrate the robustness of our framework through the emergence of tactical boxing in simulation and its zero-shot transfer to physical humanoid robots (~\autoref{fig:realworld}).
\end{itemize}

\begin{figure*}
    \centering
    \includegraphics[width=\textwidth]{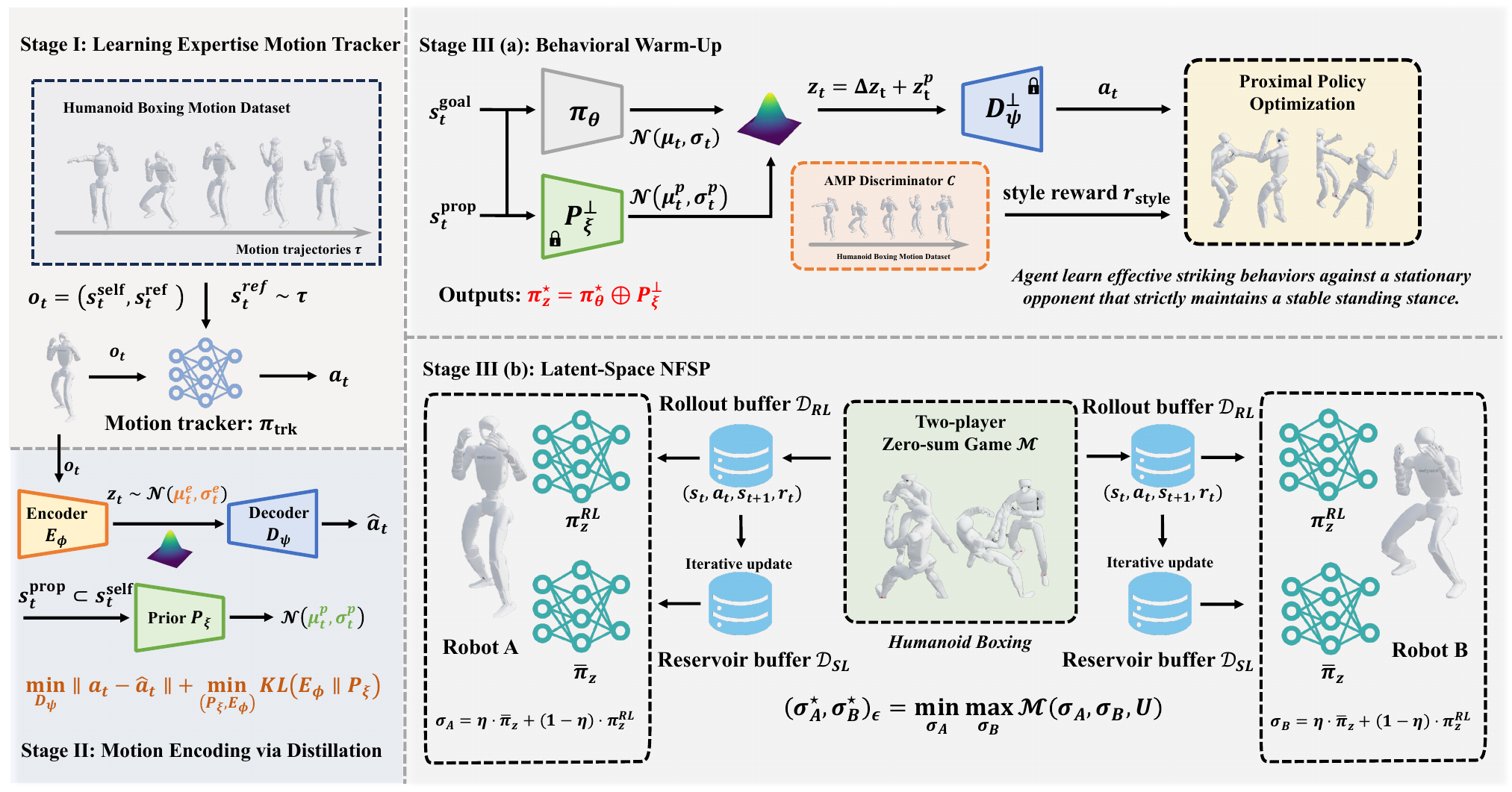}
    \vspace{-6mm}
    \caption{Overview of \framework{}. \textbf{Stage I} pretrains a motion tracker to produce physically plausible humanoid behaviors; \textbf{Stage II} compresses these behaviors into a bounded latent space for high-level control; \textbf{Stage III(a)} runs warm-up training on top of Stage II, then followed with \textbf{Stage III(b)}, a NFSP over latent-space to solve the humanoid boxing task as a two-player zero-sum game.}
    
    \label{fig:pipeline}
\end{figure*}
\section{Related Work}
\paragraph{Humanoid Motion Synthesis and Control.}
The generation of lifelike humanoid movement has evolved from purely kinematic trajectory synthesis to physics-aware control policies.
Kinematic Motion Synthesis historically relied on Inverse Kinematics (IK) and space-time optimization to retarget motion capture data onto robotic morphologies while satisfying joint and geometric constraints~\cite{gleicher1998retargetting, yamane2004synthesizing}. While deep generative models have significantly improved the diversity of synthesized trajectories~\cite{holden2016deep, zhang2023generating, petrovich2023tmr, yin2024tri}, these purely kinematic approaches lack dynamical awareness, often resulting in physical violations like "foot sliding" when transferred to simulators.
Physics-based control addresses the challenge of executing motions within a dynamical environment by accounting for gravity, inertia, and contact forces. Classical robot learning methods utilize Zero Moment Point (ZMP) stability and Model Predictive Control (MPC) to ensure bipedal balance~\cite{kajita2003biped, kuindersma2016optimization}. Modern reinforcement learning (RL)~\cite{sutton1998reinforcement} has advanced this through tracking-based imitation~\cite{peng2018deepmimic} and adversarial stylistic regularization~\cite{peng2021amp}. Crucially, the development of universal motion trackers~\cite{chen2025gmt, yin2025unitracker, zhang2025track} has enabled the learning of unified policies capable of tracking a diverse repertoire of human motions. While these universal trackers provide a robust foundation for executing a vast library of skills, they primarily focus on high-fidelity imitation of pre-defined trajectories. In contrast, our work builds upon universal tracking but extends it to a reactive, multi-agent setting, where the controller must maintain physical stability while simultaneously adapting to the non-stationary perturbations of an opponent.

\paragraph{Skill-based Latent Representation Learning.}
To manage the complexity of humanoid action spaces, recent research has shifted toward learning low-level skills in latent manifolds. PULSE~\cite{luo2023universal} establishes that distilling vast motion capture datasets into latent embeddings allows high-level policies to invoke complex behaviors via a simplified behavioral interface. Further refinements like CALM~\cite{tessler2023calm} have improved the semantic diversity and controllability of these latent spaces. However, these manifolds are typically treated as unbounded Euclidean spaces where unconstrained exploration can still lead to out-of-distribution actions that cause physical collapse. We depart from this by imposing topological regularization—projecting the latent space onto a unit hypersphere~\cite{peng2022ase, davidson2018hyperspherical}. This geometric constraint ensures that the strategic search remains bounded within a manifold of physically plausible maneuvers, effectively isolating the high-level policy from the dynamical instabilities of the robot's joints.
In this work, we leverage this topologically constrained manifold to decouple strategic reasoning from dynamics control, enabling the efficient evolution of complex competitive behaviors that would be intractable in raw observation spaces

\paragraph{Multi-Agent RL and Fictitious Self-Play.}
Self-play has emerged as a principled paradigm for learning robust policies in competitive settings by framing training as an implicit game-solving process. Classical fictitious play~\cite{brown1951iterative} iteratively best-responds to the empirical distribution of past opponents and is known to converge to a Nash equilibrium in certain classes of games. Modern adaptations of fictitious self-play (FSP) integrate reinforcement learning with historical policy averaging, achieving strong empirical performance in large-scale imperfect-information games such as poker~\cite{heinrich2016deep}. Subsequent extensions have explored continuous control and function approximation, showing that approximate best-response dynamics can converge to approximate Nash equilibria under bounded strategy spaces and smoothness assumptions~\cite{mertikopoulos2019learning}. Despite these advances, applying self-play directly in high-dimensional MARL often leads to cyclic strategic oscillations and training instabilities, particularly when policies operate in unconstrained continuous action spaces~\cite{lanctot2017unified}. Recent works have highlighted that constraining the effective strategy space—through regularization, policy parameterization, or latent abstractions—can significantly improve convergence behavior and stability in competitive learning~\cite{balduzzi2018mechanics}. From a game-theoretic perspective, restricting policies to a compact, continuous strategy set ensures the existence of mixed-strategy Nash equilibria and supports convergence guarantees for no-regret and fictitious-play-style dynamics~\cite{rosen1965existence, ramazi2020global}.


Our work builds on this line of research by situating fictitious self-play within a bounded latent strategy manifold. By conducting the competitive learning process in a compact hyperspherical latent space rather than the raw motor command space, we effectively limit the strategic degrees of freedom available to each agent. This design mitigates non-stationarity arising from unbounded policy updates and aligns the training dynamics with theoretical convergence results for approximate Nash equilibria in continuous games with compact strategy sets.


\section{Method}
We propose a hierarchical competitive learning framework that enables fully autonomous humanoid boxing through structured decision making in a continuous latent strategy space. The interaction between two humanoid agents is formalized as a two-player zero-sum Markov game, where each agent selects high-level motion intents rather than direct motor commands. A shared expert motion decoder executes these intents while preserving physical feasibility and human-like behaviors. Training proceeds through a three-stage pipeline: learning a robust low-level motion tracker from human demonstrations (Section~\ref{sec:stage1}), distilling motion primitives into a bounded continuous latent space that serves as the strategic action space (Section~\ref{sec:stage2}), and progressively introducing competition via behavioral warmup (Section~\ref{sec:stage3a}) and Latent-Space Neural Fictitious Self-Play (LS-NFSP) (Section~\ref{sec:stage3b}). This design decouples strategic reasoning from low-level control, stabilizes multi-agent learning and facilitates convergence toward approximate Nash equilibria in high-dimensional humanoid combat scenarios.

\subsection{Two-player Zero-sum Markov Games}
\label{sec:two_player_zero_sum_games}
We define a two-player zero-sum Markov Game as $\mathcal{M}=\langle \mathcal{I}, \mathcal{S}, \{\mathcal{O}\}^{i \in \mathcal{I}}, \{\mathcal{Z}\}^{i \in \mathcal{I}}, \mathcal{P}, \{\mathcal{U}\}^{i\in \mathcal{I}}, \gamma \rangle$. $\mathcal{I}=\{1,2\}$ the set of players, $\mathcal{S}$ the state space, $\{\mathcal{O}\}^{i \in \mathcal{I}}$ the set of player observation spaces. We assume the existence of a (possibly implicit) mapping $\mathcal{F}: \mathcal{O}^1 \times \cdots \times \mathcal{O}^{\vert \mathcal{I} \vert} \rightarrow \mathcal{S}$ that relates joint observations to the underlying environment state. $\{\mathcal{Z}\}^{i \in \mathcal{I}}$ the set of player action spaces where $\mathcal{Z}^i \subset \mathbb{R}^d$ and $d$ the dimension size, $\mathcal{P}: \mathcal{S} \times \mathcal{Z}^1\times\cdots\times\mathcal{Z}^{\vert \mathcal{I} \vert} \rightarrow \Delta(\mathcal{S})$ the transition function, $\{\mathcal{U}\}^{i \in \mathcal{I}}$ the set of player utility functions and $u^i \in \mathcal{U}^i: \mathcal{S} \times \mathcal{Z}^1 \times \cdots \times \mathcal{Z}^{\vert \mathcal{I} \vert} \rightarrow \mathbb{R}$, $\gamma \in [0, 1)$ the discounted factor. Let each player $i\in \mathcal{I}$ adopt a policy $\pi^i_z: \mathcal{O}^i \rightarrow \Delta(\mathcal{Z}^i)$, and the joint policy by $\boldsymbol{\pi}_z=(\pi^1_z,\pi^2_z)$. The expected return for player $i$ under joint policy $\boldsymbol{\pi}$ from state $s \in \mathcal{S}$ satisfies
\begin{equation}
\label{eq:motion_value_func}
V^i_{\boldsymbol{\pi}_z}(s) = \mathbb{E}_{z^i_t\sim\pi^i_z(\cdot\mid o^i_t)}\left[\sum_{t=0}^{\infty}\gamma^t\,u^i(s_t, z^i_t,z^{-i}_t)\;\bigg|\;s_0=s\right],
\end{equation}
where $-i$ indicates the players except for player $i$.
In the zero-sum setting $u^1=-u^2$, the learning objective for each player $i$ is to find an optimal policy $\pi^{i,*}_z$ that maximizes its expected return against the worst-case opponent policy $\pi^{-i}_z$, i.e., to solve the following max–min optimization:
\begin{equation}
\label{eq:max_min_optimization}
\pi^{i,*}_z\in\arg\max_{\pi^i_z}\,\min_{\pi^{-i}_z}J^i(\boldsymbol{\pi}_z),
\end{equation}
where $J^i(\boldsymbol{\pi}_z):=\mathbb{E}_{s\sim p(s)}\left[V^i_{\boldsymbol{\pi}_z}(s)\right]$.
When all players simultaneously solve the above optimization up to approximation error, the resulting joint policy corresponds to an $\epsilon$-Nash Equilibrium of the Markov game, as defined below.
\begin{definition}[$\epsilon$-Nash Equilibrium~\cite{aumann1976agreeing}]
    A joint policy $\boldsymbol{\pi}_z^\star = (\pi_z^{i,\star}, \pi_z^{-i,\star})$ is said to be an \emph{$\epsilon$-Nash Equilibrium} of a two-player zero-sum Markov game $\mathcal{M}$ if no player can unilaterally improve its expected return by deviating from its policy, given that the opponent adheres to $\boldsymbol{\pi}_z^\star$. Formally, for all players $i \in \mathcal{I}$ and any alternative policy $\pi_z^i$, it holds that
    \begin{equation}
        \label{eq:nash_equilibrium}
        J^i(\pi_z^{i,\star}, \pi_z^{-i,\star}) \ge J^i(\pi_z^{i}, \pi_z^{-i,\star}) - \epsilon.
    \end{equation}
\end{definition}

In this work, we frame the boxing task that involves two humanoid robots as a zero-sum Markov game, where a player represents a robot.
Rather than acting directly in the high-dimensional motor command space, each player selects actions in a bounded continuous latent space $\mathcal{Z}^i \subset \mathbb{R}^d$.
An action $z_t^i \in \mathcal{Z}^i$ represents a high-level motion latent distilled from a large motion corpus.
In practice, we constrain $\mathcal{Z}^i$ to be compact (specifically, $\mathcal{Z}^i = \mathbb{S}^{d-1}$), which restricts exploration to a physically plausible manifold and simplifies strategic interactions in continuous control settings.

While the preceding content formulate the proposed framework, its practical implementation requires additional design considerations. In the following, we provide a detailed description of a three-stage training pipeline that realizes a hierarchical decision-making process.

\subsection{Learning Expertise Motion Tracker}
\label{sec:stage1}
In the first stage, we learn a robust low-level controller that can faithfully track diverse humanoid boxing motions and provide a stable motor foundation for higher-level strategic learning.
The tracker $\pi_{\mathrm{trk}}(a_t \mid s_t^{\mathrm{self}}, s^{\mathrm{ref}}_t)$ is trained using human motion capture data and shared across all robots,
where $s_t^{\mathrm{self}}=(s_t^{\mathrm{prop}}, s_t^{\mathrm{priv}})$. Here, $s_t^{\mathrm{prop}}$ denotes a robot's proprioceptive state at timestep $t$ and $s_t^{\mathrm{priv}}$ represents privileged observations that are inaccessible to the robot's onboard sensors. $s^{\mathrm{ref}}_t$ specifies a reference motion goal derived from human motion capture data, which is represented as a time-indexed sequence of full-body humanoid poses and velocities extracted from motion capture dataset $\mathcal{D}_{\mathrm{motion}}$.
Each $s^{\mathrm{ref}}_t$ encodes the root position and orientation, joint angles, and corresponding joint velocities in a canonical humanoid kinematic tree.
All motions are temporally aligned and retargeted to the simulated humanoid morphology, enabling consistent tracking across diverse boxing behaviors.
The objective maximizes expected tracking rewards $r_{\mathrm{trk}}$ over reference trajectories $\tau^{\mathrm{ref}}$, encouraging high-fidelity imitation while maintaining physical plausibility,
\begin{align}
    &\pi_{\mathrm{trk}}^{*} = \arg\max_{\pi_{\mathrm{trk}}} \mathbb{E}_{s^{\mathrm{ref}} \sim \tau^{\mathrm{ref}}, \tau^{\mathrm{ref}}\sim \mathcal{D}_{\mathrm{motion}}} \left[ J_{\pi_{\mathrm{trk}}}(s^{\mathrm{ref}})\right],\\\nonumber
    &\text{where } J_{\pi_{\mathrm{trk}}}(s^{\mathrm{ref}})=\mathbb{E}_{a_t \sim \pi_{\mathrm{trk}}} \left[ \sum_{t=0}^{T} \gamma^{t} \, r_{\mathrm{trk}}(s_t^{\mathrm{self}}, a_t; s^{\mathrm{ref}}) \right].
\end{align}
The reward function $r_{\mathrm{trk}}$ quantifies the similarity between the simulated humanoid and the reference motion, comprising specific terms for pose alignment, velocity tracking, and control regularization. Details of reward design and data collection are provided in Appendix~\ref{app:stage1}.

\subsection{Encoding Motion via Topological Latent Distillation}
\label{sec:stage2}
To enable learnable strategic control, we project the motion space into a continuous latent space $\mathcal{Z}$.
This is achieved via a teacher-student distillation framework~\cite{ross2011reductionimitationlearningstructured} consisting of an encoder $E_\phi$, a decoder $D_\psi$, and a state-conditioned latent prior $P_\xi$.
The encoder maps observations $o_t=(s_t^{\mathrm{self}}, s^{\mathrm{ref}}_t)$ to a distribution of latent codes $z_t$, i.e.,
\begin{equation} 
z_t \sim E_{\phi}(\cdot \mid o_t).
\end{equation}
In our implementation, we model $E_{\phi}(z_t)$ as diagonal Gaussian $\mathcal{N}(z_t \mid \mu^e_t, \sigma^e_t)$. Then the decoder is trained to reconstructs the teacher’s actions $a_t \sim \pi_{\mathrm{trk}}(\cdot \mid o_t)$ conditioned on $(s^{\mathrm{prop}}_t, z_t)$ as
\begin{equation}
    D^{\star}_{\psi} = \arg\min_{D_{\psi}} \parallel a_t - \hat{a}_t\parallel,
\end{equation}
where $\hat{a}_t \sim D_{\psi}(\cdot \mid s^{\mathrm{prop}}_t,z_t)$.
Acknowledging the inherent state-dependency of motion generation, we explicitly learn a state-conditioned latent prior, denoted as $P_{\xi}(z_t\mid s^{\mathrm{prop}}_t)=\mathcal{N}(z_t \mid \mu^p_t,\sigma^p_t)$. This prior serves to model valid transitions from the current proprioceptive state and constrains the encoder $E_\phi(z_t \mid o_t)$ via KL-regularization~\cite{hershey2007approximating}, thereby preventing posterior collapse, i.e.,
\begin{equation}
    \label{eq:learning_prior}
    (E^{\star}_{\phi}, P^{\star}_\xi) = \arg\min_{(E_\phi, P_\xi)} D_{KL}\left(E_\phi(z_t \mid o_t) \parallel P_\xi(z_t \mid s^{\mathrm{prop}}_t)\right).
\end{equation}

A key design choice is to enforce the latent space $\mathcal{Z}$ to be bounded and continuous. To this end, rather than allowing unconstrained latent representations, we normalize latent codes onto a compact manifold, which restricts motion commands to a physically plausible set and prevents out-of-distribution execution by the decoder. Under this formulation, each latent action $z\in\mathcal{Z}$ fully specifies an motion decision, and we therefore model $\mathcal{Z}$ as the effective strategy space of the induced two-player game in Section~\ref{sec:two_player_zero_sum_games}. In implementation, we let $z$ satisfies the requirements with normalization as
\begin{equation}
    \hat{z} = \mathrm{Normalize}(z)=\frac{E_\phi(o)}{\| E_\phi(o) \|_2},
\end{equation}
thereby constraining the latent representation of $z$ to lie on the unit hypersphere. The total training objective is detailed in Appendix~\ref{app:stage2}.

\subsection{Behavioral Warmup with Adversarial Priors}
\label{sec:stage3a}
Considering direct competitive self-play from scratch is unstable due to the lack of basic tactical competence. We therefore introduce a behavioral warmup stage, where agents learn effective striking behaviors against a stationary opponent that strictly maintains a stable standing stance. A learnable residual policy $\pi_{\theta}(\cdot \mid s^{\mathrm{goal}}_t)$ is defined as outputs residual latent commands $\Delta z_t$ over a fixed behavioral prior $P^{\perp}_{\xi}$ from the second stage, ensuring stable and human-like motions.
Then, the complete form of $\pi_{z}$ is
\begin{equation}
    \pi_z(\cdot \mid s^{\mathrm{prop}}, s^{\mathrm{goal}}) = \pi_{\theta}(\cdot \mid s^{\mathrm{goal}}) \oplus P^{\perp}_{\xi}(\cdot \mid s^{\mathrm{prop}}_t),
\end{equation}
and the corresponding generation of $z_t$ is formulated as
\begin{equation}
    z_t=\mathrm{Normalize}(\Delta z_t + z^p_t),
\end{equation}
where $\Delta z_t \sim \pi_{\theta}(\cdot \mid s^{\mathrm{goal}})$ and $z^p_t \sim P^{\perp}_\xi(\cdot \mid s^{\mathrm{prop}})$.

To prevent degeneration of motion quality, we regularize the warmup training with AMP~\cite{peng2021amp}. Specifically, a discriminator $C$ is trained to distinguish motion transitions generated by the agent operating via the frozen decoder $D_\psi$ from reference motion data, providing a style-consistency reward $r_{\mathrm{style}}$. Thus, the learning objective is to maximize the expected hybrid return:
\begin{equation}
    \pi_z^\star = \arg\max_{\pi_{z}} \mathbb{E}_{\tau \sim (\pi_z, D_\psi)} \left[ \sum_{t=0}^{T} \gamma^t R(s_t,z_t) \right],
\end{equation}
where $R(s_t,z_t)=w_{\mathrm{task}} \cdot r_{\mathrm{task}}(s_t) + w_{\mathrm{style}} \cdot r_{\mathrm{style}}(s_t, s_{t+1})$, $\tau$ denotes the trajectory induced by the policy $\pi_z$ and the frozen decoder $D_\psi$. $w_{\mathrm{task}}$ and $w_{\mathrm{style}}$ is the weighting coefficient. The detail of the task observation $s_t^{\mathrm{goal}}$, the task reward definition and the AMP practice is illustrated in the Appendix~\ref{app:stage3a}.

\subsection{Latent-Space Neural Fictitious Self-Play}
\label{sec:stage3b}
Building on the behavioral warmup stage, we adopt Latent-Space Neural Fictitious Self-Play (LS-NFSP) to enable competitive co-evolution. Our approach inherit the core principles of Neural Fictitious Self-Play (NFSP)~\cite{heinrich2016deep},  while implementing them over a structured latent action space $\mathcal{Z}$ rather than the raw motor space. In LS-NFSP, each player is controlled by an independent LS-NFSP agent that learns through simultaneous self-play interactions with its opponents. An LS-NFSP agent interacts with its opponents and records its experience into two buffers, a reinforcement learning dataset $\mathcal{D}_{RL}$ that stores transitions and a supervised learning dataset $\mathcal{D}_{SL}$ that stores its own best-response behaviors. These two datasets are treated distinctly, corresponding to reinforcement learning and supervised learning objectives, respectively. For each agent, we train an RL policy $\pi^{RL}_{z}$ via PPO~\cite{schulman2017proximalpolicyoptimizationalgorithms} on samples from $\mathcal{D}_{RL}$. In parallel, we train an average policy $\bar{\pi}_{z}$ from $\mathcal{D}_{SL}$ with supervised learning, to imitate the agent's historical best-response behaviors. Concretely, over learning iterations $k = 1, \dots, K$, the supervised dataset is incrementally updated as
\begin{equation}
    \mathcal{D}^k_{SL} = \mathcal{D}^{k-1}_{SL} \cup \mathcal{D}^k_{RL},\forall k = 1,\dots,K.
\end{equation}
Action selection over latent space $\mathcal{Z}$ follows an mixture strategy with $\eta \in [0, 1)$.
Formally, the mixed strategy is defined as
\begin{equation}
    \sigma=\eta \cdot \bar{\pi}_z + (1 - \eta) \cdot \pi^{RL}_z,
\end{equation}
and latent actions are sampled as $z_t \sim \sigma(\cdot \mid s^{\mathrm{prop}}_t, s^{\mathrm{goal}}_t)$. Under this formulation, $\pi^{RL}_z$ represents an approximate best response to the opponents’ mixed strategies $\sigma^{-i}$, while $\bar{\pi}_z$ approximates the agent’s long-run average strategy. We summarize the pseudo algorithm of LS-NFSP in Appendix~\ref{app:stage3b}.

\subsection{Theoretical Analysis}
\label{sec:theoretical_analysis}
To analyze learning dynamics in our induced latent game, we first explicitly state a set of modeling assumptions that reflect both our hierarchical motion abstraction and the mathematical requirements for grounding equilibrium analysis in continuous strategy domains.
\begin{assumption}[Bounded Continuous Latent Space]
    \label{assump:bounded_strategy}
    Each agent's high-level policy $\pi_z$ is built on latent space $\mathcal{Z}$, where $\mathcal{Z} \subset \mathbb{R}^d$ is a compact continuous set. In our implementation, $\mathcal{Z}$ is the unit hypersphere $\mathbb{S}^{d-1}$, ensuring boundedness of the latent space.
\end{assumption}
\begin{assumption}[Well-Defined and Bounded Payoffs]
    \label{assump:bounded_payoff}
    The interaction between agents induces an effective two-player zero-sum game with payoff functions $u_i(z_i, z_{-i})$ for arbitrary $s\in\mathcal{S}$, where the expectation is taken over environment stochasticity and low-level policy execution. We assume the payoff functions are bounded and measurable over $\mathcal{S}\times\mathcal{Z}\times\mathcal{Z}$.
\end{assumption}
\begin{assumption}[Approximate Best-Response Updates]
    \label{assump:approx_br}
    At each iteration, each agent updates its latent strategy by approximately best-responding to the empirical distribution of its opponent's past strategies. The approximation error introduced by finite-horizon rollouts and stochastic optimization is assumed to be bounded.
\end{assumption}
Assumption~\ref{assump:bounded_strategy} and \ref{assump:bounded_payoff} are motivated by Glicksberg's existence theorem~\cite{glicksberg1952further}, which ensures that a mixed-strategy Nash equilibrium exists when each player’s strategy set is non-empty and compact, and the payoff functions are continuous. Beyond these structural conditions on the strategy space, the analysis of self-play dynamics further requires assumptions on the learning process itself. In our framework, high-level policies are optimized via NFSP, which alternates between computing approximate best responses and updating average strategies. We therefore abstract this optimization procedure as an approximate best-response dynamic in the induced latent game (Assumption~\ref{assump:approx_br}). Together, these assumptions establish a coherent analytical setting in which the latent game admits an equilibrium and we can derive approximate convergence results presented below.
\begin{proposition}
    \label{prop:latent_fsp}
    Under Assumptions~\ref{assump:bounded_strategy} to ~\ref{assump:approx_br}, the induced latent interaction admits the standard regularity conditions under which FSP dynamics in continuous zero-sum game are known to approach approximate Nash Equilibrium. Consequently, LS-NFSP will converge to an $\epsilon$-Nash Equilibrium of the induced game, up to bounded approximation error. 
\end{proposition}
\begin{proof}
    See~\autoref{app:proof_of_latent_fsp}.
\end{proof} 

\section{Experiment}
Policy training is performed in Isaac Lab~\cite{nvidia2025isaaclabgpuacceleratedsimulation} on the NVIDIA Omniverse platform, leveraging large-scale parallel simulation for efficient data collection. The simulated platform is the Unitree G1 humanoid~\cite{unitreeg1}, featuring 29 degrees of freedom. All experiments are trained on a single NVIDIA RTX 4090 GPU. To improve robustness and encourage the emergence of transferable tactical behaviors, we apply domain randomization~\cite{tobin2017domainrandomizationtransferringdeep} to contact friction, link masses, and actuator gains. Our evaluation is designed to answer the following questions: \textit{(i) does the proposed method achieve superior converged performance compared to relevant baselines?} and \textit{(ii) do the introduced components---including the behavioral warm-up, and AMP---contribute positively to overall performance?}


\begin{table}[t]
\centering
\caption{Win Rate (\%) of Cross-Playing. Each item the win rate of the row agent against the column agent, averaged over 20 bouts.}
\label{tab:win_rate_matrix}
\setlength{\tabcolsep}{3.5pt} 

\resizebox{\columnwidth}{!}{ 
\begin{tabular}{l|cccccccc}
\toprule
\textbf{Agent \textbackslash \ Opp} & \textbf{(1)} & \textbf{(2)} & \textbf{(3)} & \textbf{(4)} & \textbf{(5)} & \textbf{(6)} & \textbf{(7)} & \textbf{(8)} \\ \midrule
(1) \textbf{Ours LS-NFSP}   & - & \textbf{68.52} & \textbf{76.24} & \textbf{82.41} & \textbf{84.47} & \textbf{95.38} & \textbf{98.50} & \textbf{100.0} \\
(2) Fictitious SP            & - & - & 62.35 & 71.84 & 75.16 & 92.11 & 96.18 & 99.42 \\
(3) Naive SP             & - & - & - & 68.45 & 70.38 & 88.61 & 94.52 & 98.53 \\
(4) LS-NFSP w/o AMP                & - & - & - & - & 62.12 & 82.14 & 88.37 & 95.18 \\
(5) PPO-Only              & - & - & - & - & - & 45.24 & 77.86 & 82.11 \\
(6) Static-Target Specialist          & - & - & - & - & - & - & 60.33 & 85.58 \\
(7) Naive SP w/o Warmup             & - & - & - & - & - & - & - & 94.17 \\
(8) 29Dof Action-Space SP           & - & - & - & - & - & - & - & - \\ \bottomrule
\end{tabular}
}
\end{table}

\begin{table}[t]
\centering
\caption{Comparison between Action-Space (29-DOF) vs. Latent-Space (Ours) on Tactical and Physical proficiency.}
\label{tab:core_comparison}
\setlength{\tabcolsep}{3.5pt} 
\renewcommand{\arraystretch}{1.2}
\resizebox{\columnwidth}{!}{ 
\begin{tabular}{lcccc}
\toprule
\textbf{Metric} & \textbf{29Dof Action-Space SP} & \textbf{Ours LS-NFSP} \\ \midrule
\rowcolor[HTML]{F2F2F2}\textbf{(a)Tactical Metrics} & & & \\
Offensive Landing Rate ($\eta_{hit}$) $\uparrow$ & $0.142 \pm 0.05$ & $\mathbf{0.685 \pm 0.03}$ \\
Engagement Rate ($ER$) $\uparrow$   & $0.315 \pm 0.08$ & $\mathbf{0.824 \pm 0.02}$ \\ \hline
\rowcolor[HTML]{F2F2F2}\textbf{(b)Physical Metrics} & & & \\
Base Stability ($BOS$) $\uparrow$ & $0.418 \pm 0.12$ & $\mathbf{0.942 \pm 0.01}$ \\
Torque Smoothness ($TS_{\tau}$) $\downarrow$  & $7.452 \pm 1.211$ & $\mathbf{0.930 \pm 0.150}$ \\ \bottomrule
\end{tabular}
}
\end{table}
\subsection{Evaluation Metrics}
\label{sec:metrics}
To answer the above question, we propose some metrics to quantify the evaluation of LS-NFSP and baselines from three primary dimensions, which is motivated by some principles in professional boxing training and robust robotics control, as listed as follows.


\paragraph{Tactical Proficiency.}
The tactical proficiency concentrates on evaluating the offensive precision and active engagement capabilities. Thus, we introduces two metrics as \textit{Offensive Landing Rate} ($\eta_{\mathrm{hit}}$) and \textit{Engagement Rate} ($ER$). $\eta_{\mathrm{hit}}$ is identified by the proportion of contacts exceeding the force threshold to the total number of offensive attempts.
$ER$ acts as a rigorous measure of an agent's ability to maintain a viable combat posture, effectively penalizing tactical avoidance or non-confrontational movement patterns, the higher the better, and it is calculated as the proportion of the episode duration in which the agent concurrently satisfies two spatial conditions: (1) keep a distance within an effective striking range and (2) ensuring a facing alignment that exceeds a predefined threshold.

\paragraph{Physical Stability.}
Physical stability is assessed through \textit{Base Orientation Stability} ($BOS$) and \textit{Torque Smoothness} ($TS_{\tau}$). To ensure a positive correlation with control quality, $BOS$ is defined using an exponential kernels on the angular deviation: $BOS = \mathbb{E} \left[ \exp(-\| \mathbf{g}_{\mathrm{base}} - \mathbf{g}_{\mathrm{world}} \|^2) \right]$, where $\mathbf{g}_{\mathrm{base}}$ and $\mathbf{g}_{\mathrm{world}}$ the gravity vector in the robot's local frame and the global frame, respectively. A higher $BOS$ indicates a superior ability to maintain an upright posture under competitive perturbations. $TS_{\tau}$ measures the magnitude of high-frequency oscillations in motor commands, calculated as the mean absolute change in joint torques across consecutive control steps: $TS_{\tau} = \mathbb{E} \left[ \| \tau_{t} - \tau_{t-1} \| \right]$. A lower $TS_{\tau}$ value represents more hardware-friendly execution and reduced mechanical wear.

\paragraph{Stylistic Authenticity.}
The Stylistic Authenticity of the emergent behaviors is evaluated via qualitative visual inspection. Instead of relying on a singular numerical score, we provide a comparative analysis of simulation snapshots to assess the human-likeness of the motions. This includes identifying key boxing maneuvers such as slips, counters, and rhythmic footwork, while ensuring the absence of unnatural joint configurations or physically unsustainable postures that often characterize non-constrained reinforcement learning policies.




\begin{table}[t]
\centering
\caption{Tactical performance comparison focusing on Offensive Landing Rate and Engagement Rate. All metrics are evaluated against the Naive Self-Play (Latent). LS-NFSP outperforms all the baselines.}
\label{tab:tactical_results}
\setlength{\tabcolsep}{3.5pt} 
\renewcommand{\arraystretch}{1.2}
\resizebox{\columnwidth}{!}{ 
\begin{tabular}{lcc}
\toprule
\textbf{Methods} & \textbf{Offensive Landing Rate} $\uparrow$ & \textbf{Engagement Rate} $\uparrow$ \\
\midrule
\rowcolor[HTML]{F2F2F2} \textbf{(a) Comparison with Strategic Algorithms} & & \\
PPO-Only & $0.231 \pm 0.03$ & $0.495 \pm 0.02$ \\
Naive SP & $0.350 \pm 0.04$ & $0.580 \pm 0.05$ \\
Fictitious SP & $0.420 \pm 0.03$ & $0.650 \pm 0.04$ \\
\textbf{Ours LS-NFSP} & $\mathbf{0.685 \pm 0.03}$ & $\mathbf{0.824 \pm 0.02}$ \\
\midrule
\rowcolor[HTML]{F2F2F2} \textbf{(b) Impact of Latent Space} & & \\
29Dof Action-Space SP & $0.142 \pm 0.05$ & $0.315 \pm 0.08$ \\
\textbf{Ours LS-NFSP} & $\mathbf{0.685 \pm 0.03}$ & $\mathbf{0.824 \pm 0.02}$ \\
\midrule
\rowcolor[HTML]{F2F2F2} \textbf{(c) Impact of Curriculum Design} & & \\
Static-Target Specialist & $0.210 \pm 0.04$ & $0.450 \pm 0.06$ \\
SP w/o Warmup & $0.050 \pm 0.02$ & $0.120 \pm 0.05$ \\
\textbf{Ours LS-NFSP} & $\mathbf{0.685 \pm 0.03}$ & $\mathbf{0.824 \pm 0.02}$ \\
\midrule
\rowcolor[HTML]{F2F2F2} \textbf{(d) Impact of Stylistic Constraints} & & \\
LS-NFSP w/o AMP & $0.490 \pm 0.03$ & $0.720 \pm 0.05$ \\
\textbf{Ours LS-NFSP} & $\mathbf{0.685 \pm 0.03}$ & $\mathbf{0.824 \pm 0.02}$ \\
\bottomrule
\end{tabular}
}
\end{table}

\subsection{Cross-Playing Evaluation}
\label{sec:cross_play}
We answer the first question by conducting cross-play between LS-NFSP and involved baselines described in~\autoref{app:exp}. \autoref{tab:win_rate_matrix} presents the win rates over all the involved methods, which shows that LS-NFSP outperforms all the baselines. The most critical insight is the superiority of performing competitive self-play within a structured latent manifold. We notice that the LS-NFSP agent dominates 29Dof Action-Space SP agent with a 100.00\% win rate and also performs superior tactical and physical proficiency (\autoref{tab:core_comparison}), proving that decoupling balance maintenance from tactical exploration is a fundamental prerequisite for competitive learning in high-dimensional humanoid combat. And for the comparison between self-play methods with latent space, LS-NFSP also outperforms the other methods, in the result of Fictitious SP of 68.52\% and Naive SP of 76.24\%, confirming that the integration of NFSP effectively mitigates policy cycling to produce a more robust best response. The necessity of each training stage is further validated by the substantial performance gaps against Direct-SP of 98.50\% and Static-Target Specialist of 95.38\%, which fail to respectively overcome reward sparsity and reactive dynamic opponents. Additionally, the No-AMP variant exhibits reduced tactical efficiency of 82.41\% win rate against Ours due to its non-standardized movement patterns. Finally, the significant win rate against the PPO-Only baseline of 84.47\% underscores that competitive interaction is the indispensable driver of emergent combat behaviors.

\subsection{Ablation Study}
\label{sec:comparative_analysis}
To answer the second question, we conduct ablation study from two perspectives, i.e., whether the strategy prior, warm-up and AMP can benefit the policy training.
\begin{figure}
    \centering
    \vspace{-2mm}
    \includegraphics[width=\columnwidth]{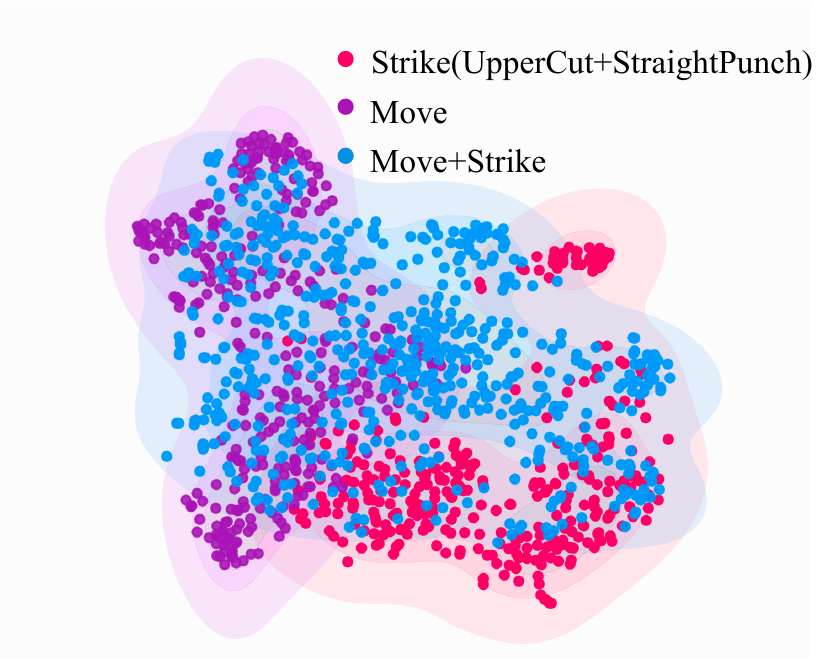}
    \caption{t-SNE visualization of the 32-dimensional latent manifold, illustrating the structured semantic clustering of combat primitives and their topological relationships that facilitate stable, compositional behavioral transitions in the hierarchical LS-NFSP framework.}
    \vspace{-3mm}
    \label{fig:tsne_manifold}
\end{figure}

\paragraph{Strategy Priors.}
\autoref{tab:tactical_results}~(a) underscores the necessity of both competitive strategy priors in achieving elite boxing proficiency. LS-NFSP demonstrates superior tactical performance by achieving $\eta_{\mathrm{hit}}=0.685$ and $ER=0.824$, which significantly outperforms baselines.

\paragraph{Warm-Up Training.}
As shown in \autoref{tab:tactical_results}(c), the necessity of leveraging warm-up is further evidenced by the collapse of the Direct-SP variant. Despite competitive access, it achieves a negligible $\eta_{\mathrm{hit}}=0.050$, proving its inability to overcome the inefficiency due to reward sparsity settings.

\paragraph{Adversarial Motion Prior.}
We further evaluate the impact of stylistic constraints rigorously as in \autoref{tab:tactical_results}(d). While the variant achieves $ER=0.720$, its $\eta_{\mathrm{hit}}$ drops sharply to 0.490. This discrepancy reveals that, without AMP, the agent can approach the opponent but executes erratic strikes that difficult to contact opponent effectively, as \autoref{fig:mujoco_vis} demonstrates that the policy trained without AMP lacks physical authenticity. Ultimately, these results prove that the synergy of NFSP-driven exploration and AMP-guided movement is essential for producing a controller that is both strategically dominant and physically authentic.

\begin{figure}
    \centering
    \includegraphics[width=\columnwidth]{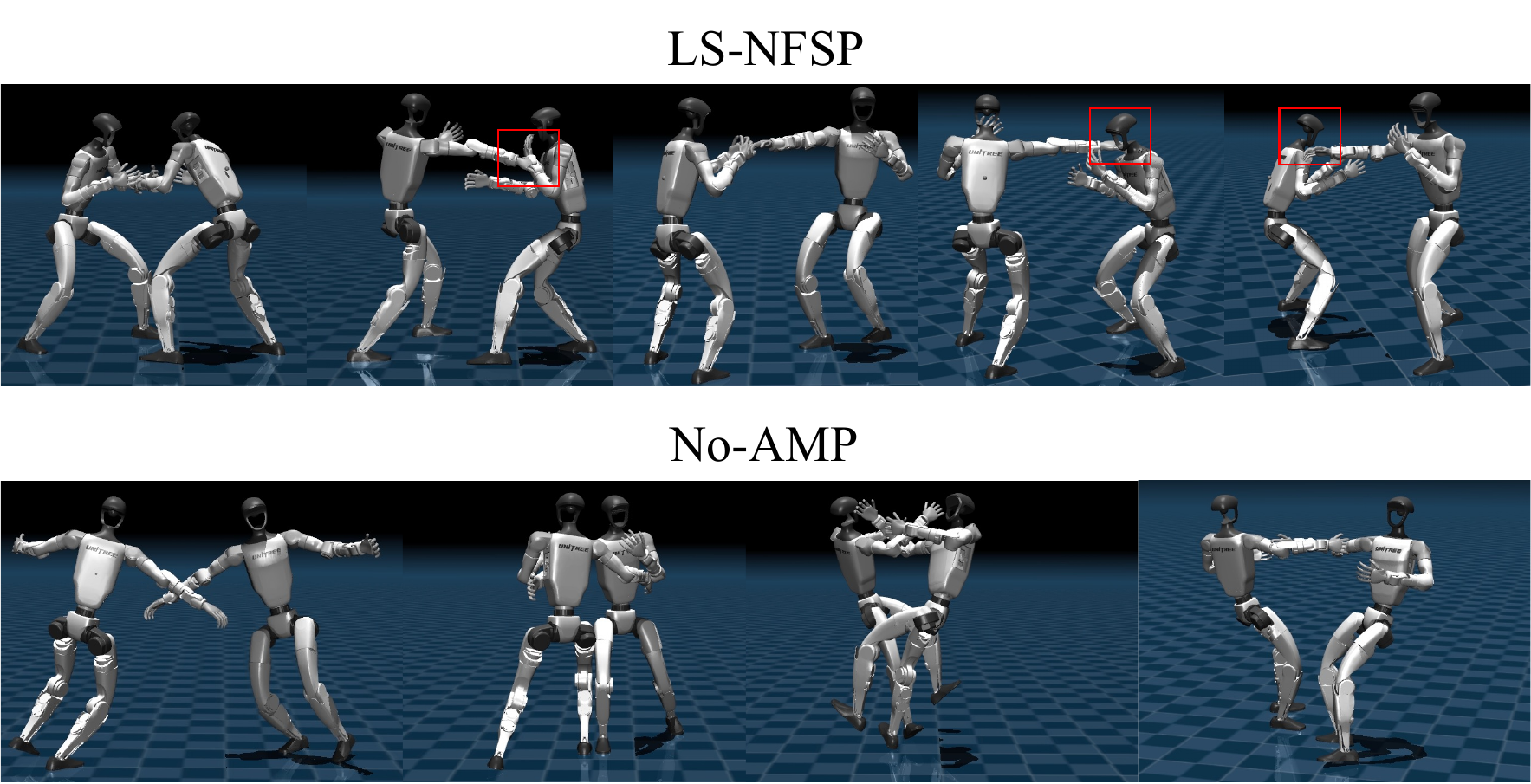}
    \caption{Mujoco visualization of the LS-NFSP and the one without AMP.}
    \vspace{-6mm}
    \label{fig:mujoco_vis}
\end{figure}

\subsection{Analysis of the Learned Latent Space}
\label{sec:tsne_analysis}
To investigate the underlying structure of the latent skills, we perform a joint t-SNE~\cite{maaten2008visualizing} analysis on the generated embeddings, as illustrated in \autoref{fig:tsne_manifold}. The visualization reveals a highly structured manifold with a distinct topological arrangement: there is a clear separation between the pure \textit{Move} (purple) and pure \textit{Strike} (pink) clusters. This demarcation indicates that the latent space successfully disentangles stable locomotion dynamics from high-momentum offensive maneuvers. Crucially, \textit{Move+Strike} (blue) embeddings occupy the intermediate space, exhibiting significant overlap with both \textit{Move} and \textit{Strike} regions. This dual overlap suggests that the \textit{Move+Strike} mode functions as a topological bridge, seamlessly connecting the two distinct primitives. Such a structure provides empirical evidence of the latent space's compositional nature, enabling the LS-NFSP framework to smoothly interpolate between navigation and combat. This allows the agent to initiate attacks dynamically while moving, without suffering from the discontinuity or instability typically associated with switching between disjoint control modes.
\section{Conclusion}
We propose a hierarchical competitive framework for humanoid interaction that integrates latent-space control with Neural Fictitious Self-Play (LS-NFSP). By decoupling high-level tactical strategy from low-level balance control via a pre-trained latent manifold, our approach ensures physical stability while enabling the emergence of sophisticated boxing maneuvers like slip-and-counter tactics. Experimental results from an $8 \times 8$ cross-play tournament, supported by qualitative snapshots and t-SNE analysis, demonstrate that LS-NFSP consistently outperforms baselines in both competitive win rate and stylistic authenticity. Ultimately, this work proves that the synergy of hierarchical latent representations and competitive self-play provides a robust and scalable solution for achieving tactical intelligence in high-dimensional humanoid robotics.


\bibliography{example_paper}
\bibliographystyle{icml2026}

\newpage
\appendix
\onecolumn
\section{Learning an Expert Motion Tracker}
\label{app:stage1}
This section supplements Section~\ref{sec:stage1} with implementation details for learning an expert motion tracker, including motion dataset collection and processing, as well as the reinforcement learning setup used to train the tracking policy.

\subsection{Motion Dataset Collection and Processing}
We collect a diverse set of human boxing motions recorded by professional boxers using an Xsens motion capture system~\cite{xsens_mvn}. The dataset consists of 46 original motion clips with a total duration of approximately 14 minutes, captured at a frequency of 50 Hz. These clips cover a wide range of boxing behaviors, including offensive strikes, defensive maneuvers, footwork, and transitional movements.
To increase data diversity and improve generalization, we augment the dataset through left-right mirroring, effectively doubling the number of available motion sequences.

The captured human motions are retargeted to the Unitree G1 humanoid~\cite{unitreeg1} morphology following the Generalized Motion Retargeting (GMR) framework~\cite{araujo2025retargetingmattersgeneralmotion}. Through this retargeting process, we obtain physically realizable reference trajectories dataset $\mathcal{D}_{\mathrm{motion}}$ that preserve the stylistic characteristics and temporal structure of the original boxing motions.

\subsection{Reinforcement Learning Settings and Training Details}
\paragraph{State, Observation and Action Spaces.}
The state of the tracking policy is defined as
$s_t^{\mathrm{self}} = (s_t^{\mathrm{prop}}, s_t^{\mathrm{priv}})$,
where $s_t^{\mathrm{prop}}$ denotes the proprioceptive state, including base angular velocities, relative joint positions, and joint velocities, and $s_t^{\mathrm{priv}}$ denotes privileged information, including base linear velocities, global base position and orientation, body link poses, and contact states.
The observation $\mathbf{o}_{\mathrm{trk}} \in \mathcal{O}$ consists of three components:
(i) the proprioceptive state $s_t^{\mathrm{prop}}$;
(ii) the reference motion state $s_t^{\mathrm{ref}} \sim \tau^{\mathrm{ref}}$, which provides target full-body poses and base velocities over the current and the next $K=3$ timesteps; and
(iii) the previous action $a_{t-1}$ to promote temporal smoothness.
The action $a_t \in \mathcal{A}$ corresponds to target joint positions for the 29-degree-of-freedom Unitree G1 humanoid~\cite{unitreeg1}, which are tracked by a low-level PD controller.

\paragraph{Reward Function.}
The tracking reward $r_{\mathrm{trk}}$ is designed to encourage accurate imitation of reference boxing motions and is defined as a weighted sum of exponential tracking terms:
\begin{align*}
   r_{\mathrm{trk}}(s_t^{\mathrm{self}}, a_t;s_t^{\mathrm{ref}}) &= w_{\mathrm{p}} \cdot r_{\mathrm p}(s^{\mathrm{self}}_{t}, s^{\mathrm{ref}}_{t} ) + w_{\mathrm o} \cdot r_{\mathrm o}(s^{\mathrm{self}}_{t}, s^{\mathrm{ref}}_{t} ) + w_{\mathrm{bp}} \cdot r_{\mathrm{bp}}(s^{\mathrm{self}}_{t}, s^{\mathrm{ref}}_{t} ) \\\nonumber
    & + w_{\mathrm{bo}} \cdot r_{\mathrm{bo}}(s^{\mathrm{self}}_{t}, s^{\mathrm{ref}}_{t} ) + w_{\mathrm{lv}} \cdot r_{\mathrm{lv}}(s^{\mathrm{self}}_{t}, s^{\mathrm{ref}}_{t} ) + w_{\mathrm{av}} \cdot r_{\mathrm{av}}(s^{\mathrm{self}}_{t}, s^{\mathrm{ref}}_{t} ) - r_{\mathrm{reg}}(s^{\mathrm{self}}_t, a_t),
\end{align*}
where $\mathcal{X}={\mathrm{p,o,bp,bo,lv,av}}$ and each term is defined as
$
r_x = \exp(-\|\mathrm{err}_x\|^2 / \sigma_x^2).
$
Here, $r_{\mathrm p}$ and $r_{\mathrm o}$ measure global root position and orientation tracking errors, $r_{\mathrm{bp}}$ and $r_{\mathrm{bo}}$ measure relative body link pose errors, and $r_{\mathrm{lv}}$ and $r_{\mathrm{av}}$ penalize deviations in linear and angular velocities. The regularization term $r_{\mathrm{reg}}$ penalizes excessive action rate, joint limit violations, and undesired self-collisions or ground contacts. The weights $w_x$ and scaling factors $\sigma_x$ are provided in Table \ref{tab:trackingreward}.

\begin{table}[ht]
\centering
\caption{Reward Terms, Weights and Scaling Factor for Motion Tracker.}
\label{tab:trackingreward}
\begin{tabular}{lcc}\toprule
\textbf{Reward Terms}                               & \textbf{Weights} $w$ & \textbf{Scaling Factor} $\sigma$ \\ \midrule
Root Position $r_\mathrm{p}$          & $0.5$       & $0.3$                   \\
Root Orientation $r_\mathrm{o}$       & $0.5$       & $0.4$                   \\
Body Position $r_\mathrm{bp}$         & $1.0$       & $0.3$                   \\
Body Orientation $r_\mathrm{bo}$      & $1.0$       & $0.4$                   \\
Body Linear Velocity $r_\mathrm{lv}$  & $1.0$       & $1.0$                   \\
Body Angular Velocity $r_\mathrm{av}$ & $1.0$       & $3.14$                 \\ \bottomrule
\end{tabular}
\end{table}

\paragraph{Early Termination.}
To accelerate training and prevent the policy from exploring unstable or physically unfeasible states, we follow ~\cite{peng2018deepmimic} to implement an early termination mechanism $\mathcal{T}$. An episode is terminated if any of the following conditions are met: 1). \textbf{Pose Collapse:} The vertical distance between the robot's base and its reference height falls below a threshold of $0.25\,\text{m}$, or the base orientation deviation exceeds $0.8$ rad. 2). \textbf{End-Effector Violation:} The vertical position of key end-effectors, specifically the ankles and wrists, deviates from the reference height by more than $0.25\,\text{m}$. 3).\textbf{Temporal Limit:} The maximum episode duration of $10.0$ seconds is reached.

By combining structured reward shaping, early termination, and massive-scale parallel simulation in Isaac Lab~\cite{nvidia2025isaaclabgpuacceleratedsimulation}, the low-level tracking policy $\pi_{\mathrm{trk}}$ learns robust and physically consistent motor skills, which form the foundation for subsequent latent skill distillation.

\section{Encoding Motion via Topological Latent Distillation}
\label{app:stage2}
This section provides implementation details for the latent motion distillation framework introduced in Section~\ref{sec:stage2}


\paragraph{Distillation Architecture.}
We adopt a student–teacher paradigm to distill the low-level tracking policy $\pi_{\mathrm{trk}}$ into a compact latent interface. During training, the teacher policy is fixed, and a student model consisting of an encoder $E_\phi$, a decoder $D_\psi$, and a state-conditioned prior network $P_\xi$ is jointly optimized.
At each timestep, the encoder maps privileged observations $o_t=(s_t^{\mathrm{self}}, s_t^{\mathrm{ref}})$ to a latent distribution $q_\phi(z_t\mid o_t)$. A latent code $z_t$ is sampled and provided to the decoder, which predicts target joint actions conditioned on the proprioceptive state.



\paragraph{Training Objective.} 
The student is trained to reconstruct the teacher’s actions while learning a structured latent distribution. The reconstruction loss is defined as
\begin{equation*}
\mathcal{L}_{\mathrm{rec}} = \mathbb{E}_{s^{\mathrm{ref}}_t \sim \tau^{\mathrm{ref}}, \tau^{\mathrm{ref}}\sim \mathcal{D}_{\mathrm{motion}}}\left[\| a_t - \hat a_t\|^2\right],
\end{equation*}
where $a_t \sim \pi_{\mathrm{trk}}(\cdot \mid o_t)$, $z_t \sim E_\phi(\cdot \mid o_t)$, $\hat a_t \sim D_\psi(s_t^{\mathrm{prop}}, z_t)$ . To regularize the latent space, we impose a KL divergence between the encoder posterior and the state-conditioned prior:

\begin{equation*}
\mathcal{L}_{\mathrm{prior}} = \mathbb{E}_{s^{\mathrm{ref}}_t \sim \tau^{\mathrm{ref}},\tau^{\mathrm{ref}} \sim \mathcal{D}_{\mathrm{motion}}}\left[D_{\mathrm{KL}}\left(E_\phi(z_t\mid o_t)\|P_\xi(z_t\mid s_t^{\mathrm{prop}})\right)\right].
\end{equation*}
The prior loss ensures that the prior network captures the semantic structure of the motion primitives. During the competitive stage in self-play, this learned prior provides a critical reference for the high-level policy, where a KL penalty is applied to prevent the strategy from deviating into physically unstable regions of the action space.
The total loss is given by
\begin{equation*}
\mathcal{L}_{\mathrm{distill}} = \mathcal{L}_{\mathrm{rec}} + \lambda_{\mathrm{prior}} \mathcal{L}_{\mathrm{prior}}.
\end{equation*}
where $\lambda_{prior}=0.001$. This dual optimization ensures that the student not only recovers the motor control expertise of the teacher but also constructs a structured, prior-conditioned manifold ready for competitive co-evolution.



\section{Behavioral Warmup with Adversarial Priors}
\label{app:stage3a}

This section gives implementation details for the behavioral wamup training. We also describe how Adversarial Motion Priors (AMP) are incorporated to regularize behaviors toward the motion-capture distribution and preserve human-like boxing style.


\subsection{Goal-Oriented Observation for Autonomous Boxing}

Different from previous tracking-based task settings, the observation of the autonomous boxing task does not contain any explicit motion commands. Instead, the residual policy $\pi_{\theta}$ relies on goal-oriented spatial information to guide decision-making. The task observation $s^{\mathrm{goal}}$ consists of two complementary components: (1) offensive target information and (2) defensive threat information.

\paragraph{Offensive Target Observation.}
The offensive target observation encodes the relative positions of the ego agent’s left and right fists with respect to the opponent’s torso, represented in the ego-centric coordinate frame. Let $\mathbf{p}_l^{\mathrm{ego}}, \mathbf{p}_r^{\mathrm{ego}} \in \mathbb{R}^3$ denote the world-frame positions of the ego agent’s wrists, $\mathbf{p}_t^{\mathrm{opp}} \in \mathbb{R}^3$ the position of the opponent’s torso, and $\mathbf{q}^{\mathrm{ego}}$ the ego root orientation. The relative vectors are defined as

\begin{equation}
\mathbf{v}_{l}^{\mathrm{off}} = \mathbf{p}_l^{\mathrm{ego}} - \mathbf{p}_t^{\mathrm{opp}}, \quad
\mathbf{v}_{r}^{\mathrm{off}} = \mathbf{p}_r^{\mathrm{ego}} - \mathbf{p}_t^{\mathrm{opp}},
\end{equation}

which are further transformed into the ego-centric frame via

\begin{equation}
\tilde{\mathbf{v}}^{\mathrm{off}} = \mathbf{R}(\mathbf{q}^{\mathrm{ego}})^{-1} \mathbf{r}^{\mathrm{off}},
\end{equation}

where $\mathbf{R}(\cdot)$ denotes the rotation matrix induced by the quaternion.

This transformation removes global orientation dependency and ensures rotational invariance of the observation. As a result, the offensive component provides direct and physically interpretable geometric information for guiding accurate, distance-aware striking behaviors.

\paragraph{Defensive Target Observation.}
The defensive target observation characterizes the relative positions of the opponent’s left and right fists with respect to the ego agent’s torso, also expressed in the ego-centric frame. Let $\mathbf{p}_l^{\mathrm{opp}}, \mathbf{p}_r^{\mathrm{opp}} \in \mathbb{R}^3$ denote the opponent’s wrist positions and $\mathbf{p}_t^{\mathrm{ego}} \in \mathbb{R}^3$ the ego torso position. The defensive vectors are computed as

\begin{equation}
\mathbf{v}_{l}^{\mathrm{def}} = \mathbf{p}_l^{\mathrm{opp}} - \mathbf{p}_t^{\mathrm{ego}}, \quad
\mathbf{v}_{r}^{\mathrm{def}} = \mathbf{p}_r^{\mathrm{opp}} - \mathbf{p}_t^{\mathrm{ego}},
\end{equation}

and rotated into the local frame in the same manner:

\begin{equation}
\tilde{\mathbf{v}}^{\mathrm{def}} = \mathbf{R}(\mathbf{q}^{\mathrm{ego}})^{-1} \mathbf{v}^{\mathrm{def}}.
\end{equation}
This component enables the policy to perceive incoming attack trajectories independently of global pose, facilitating anticipatory defense, evasive maneuvers, and counter-attacking strategies. 

Together, the offensive and defensive observations form a compact yet expressive goal representation $s^{\mathrm{goal}} = (\tilde{\mathbf{v}}^{\mathrm{off}}_l, \tilde{\mathbf{v}}^{\mathrm{off}}_r,\tilde{\mathbf{v}}^{\mathrm{def}}_l, \tilde{\mathbf{v}}^{\mathrm{def}}_r)$ that supports coordinated attack--defense behaviors in autonomous humanoid boxing.

\subsection{Adversarial Motion Prior}
To preserve human-like motion quality, we follow ~\cite{peng2018deepmimic} to use a discriminator $C$ to distinguish whether the state originates from mocap dataset $\mathcal{D}_{\mathrm{motion}}$ or the agent formed by the policy $\pi_z$ and the frozen decoder $D_\psi$. Specifically, we first construct the AMP observation ${o_t}^{\mathrm{disc}}$ to better capture the style of the mocap data. The observation ${o_t}^{{\mathrm{disc}}}$ includes the local rotation and velocity of each joint, along with the angular velocities of the base. Observation transitions from the agent and those sampled from the mocap data are fed separately to the discriminator. We label the transitions from the mocap data as \textit{real} and those from the agent as \textit{fake}. The goal of the discriminator is to classify these inputs correctly: when mocap transitions are provided, the output should approach 1, whereas transitions from the agent should produce outputs closer to -1. As a result, the objective function can be formulated as follows:
\begin{equation*}
    \begin{aligned}
        \arg \min_C & \ \mathbb{E}_{d^{\mathcal{D}_{\mathrm{motion}}}({o_t}^{\mathrm{disc}},{o_{t+1}^{\mathrm{disc}}})} \left[ \left( C({o_t}^{\mathrm{disc}},{o_{t+1}^{\mathrm{disc}}}) - 1 \right)^2 \right] \\
        & + \mathbb{E}_{d^{\pi_z, D_\psi}({o_t}^{\mathrm{disc}},{o_{t+1}^{\mathrm{disc}}})} \left[ \left( C({o_t}^{\mathrm{disc}},{o_{t+1}^{\mathrm{disc}}}) + 1 \right)^2 \right].
    \end{aligned}
\end{equation*}

Since discriminators often suffer from mode collapse, we follow AMP~\cite{peng2018deepmimic} by applying a gradient penalty to the mocap data transitions to ensure stability. This can be formulated as follows:
\begin{equation*}
    \arg \min_C  \ \mathbb{E}_{d^{\mathcal{D}_{\mathrm{motion}}}({o_t}^{\mathrm{disc}},{o_{t+1}^{\mathrm{disc}}})} \left[ \left\lVert \triangledown C({o_t}^{\mathrm{disc}},{o_{t+1}^{\mathrm{disc}}}) \right\rVert^2 \right].
\end{equation*}

The final objective function to train the discriminator is formulated as follows:
\begin{equation*}
    \begin{aligned}
        \arg \min_C & \ \mathbb{E}_{d^{\mathcal{D}_{\mathrm{motion}}}({o_t}^{\mathrm{disc}},{o_{t+1}^{\mathrm{disc}}})} \left[ \left( C({o_t}^{\mathrm{disc}},{o_{t+1}^{\mathrm{disc}}}) - 1 \right)^2 \right] \\
        & + \mathbb{E}_{d^{\pi_z, D_\psi}({o_t}^{\mathrm{disc}},{o_{t+1}^{\mathrm{disc}}})} \left[ \left( C({o_t}^{\mathrm{disc}},{o_{t+1}^{\mathrm{disc}}}) + 1 \right)^2 \right] \\
        & + w_{gp}\mathbb{E}_{d^{\mathcal{D}_{\mathrm{motion}}}({o_t}^{\mathrm{disc}},{o_{t+1}^{\mathrm{disc}}})} \left[ \left\lVert \triangledown C({o_t}^{\mathrm{disc}},{o_{t+1}^{\mathrm{disc}}}) \right\rVert^2 \right],
    \end{aligned}
\end{equation*}

\subsection{Reward Formulation}
\label{sec:reward}

To encourage effective engagement, active striking, and robust defensive behaviors in humanoid combat, we design a hierarchical and compositional reward function. The overall reward integrates geometric alignment, locomotion, contact-based interaction, and terminal dominance, enabling stable learning under self-play settings.

\paragraph{Overall Objective.}
The warm-up reward is defined as a weighted combination of task-oriented and style-oriented objectives:
\begin{equation*}
r_{\mathrm{warmup}}
=
w_{\mathrm{task}} r_{\mathrm{task}}
+
w_{\mathrm{style}} r_{\mathrm{style}},
\end{equation*}
where $r_{\mathrm{style}}$ encourages motion naturalness, and $r_{\mathrm{task}}$ focuses on combat effectiveness.

\paragraph{Task-Level Decomposition.}
The task reward is formulated as
\begin{equation*}
\begin{aligned}
r_{\mathrm{task}} =
&\; w_{\mathrm{face}} r_{\mathrm{face}}
+ w_{\mathrm{vel}} r_{\mathrm{vel}} \\
&+ w_{\mathrm{dist}} r_{\mathrm{dist}}
+ w_{\mathrm{hit}} r_{\mathrm{hit}} \\
\end{aligned}
\end{equation*}
where each component captures a distinct aspect of humanoid combat behavior.

\paragraph{Hierarchical Design.}
The reward structure follows a layered principle:
(i) geometric alignment and distance shaping,
(ii) velocity-oriented locomotion, and 
(iii) contact-based striking behaviors.
This hierarchy stabilizes optimization and mitigates reward exploitation.
\paragraph{Facing Alignment Reward.}
We encourage the agent to orient its torso toward the opponent target.
Let $\mathbf{p}_e$ and $\mathbf{p}_t$ denote the ego root position and the opponent target position. The horizontal target direction in the world frame is defined as
\begin{equation*}
\mathbf{d}_w =
\frac{\pi_{xy}(\mathbf{p}_t - \mathbf{p}_e)}
{\|\pi_{xy}(\mathbf{p}_t - \mathbf{p}_e)\|},
\end{equation*}
where $\pi_{xy}(\cdot)$ projects a vector onto the horizontal plane.
Let $R_e \in SO(3)$ be the rotation matrix corresponding to $\mathbf{q}_e$. The direction expressed in the ego local frame is
\begin{equation*}
\mathbf{d}_l = R_e^{-1} \mathbf{d}_w = (d_x,d_y,d_z)^\top .
\end{equation*}
The facing reward is defined as
\begin{equation*}
r_{\mathrm{face}}
=
\exp\!\left(
-\frac{1-d_x}{\sigma_{\mathrm{face}}}
\right),
\end{equation*}
where $d_x$ represents the forward component of the target direction, and $\sigma_{\mathrm{face}}$ controls the sharpness.


\paragraph{Velocity Reward.}
We reward locomotion along the line of sight to the opponent.
Let $\mathbf{v}_e$ denote the ego root linear velocity and $\mathbf{d}_{2D}$ the normalized horizontal target direction:
\begin{equation*}
\mathbf{d}_{2D} =
\frac{(\mathbf{p}_t-\mathbf{p}_e)_{xy}}
{\|(\mathbf{p}_t-\mathbf{p}_e)_{xy}\|}.
\end{equation*}
The projected velocity along the target direction is
\begin{equation*}
v_\parallel = \mathbf{v}_{e,xy}^\top \mathbf{d}_{2D}.
\end{equation*}

Given the desired approaching speed $v_{\mathrm{tar}}$, we define the velocity error
\begin{equation*}
e_v = \max(0, v_{\mathrm{tar}} - v_\parallel),
\end{equation*}
and the reward
\begin{equation*}
r_{\mathrm{vel}}
=
\mathbb{I}[v_\parallel>0]
\exp\!\left(-\frac{e_v^2}{\sigma_{\mathrm{vel}}}\right),
\end{equation*}
where $\sigma_{\mathrm{vel}}$ determines the scaler.




\paragraph{Distance Reward.}
To prevent passive exploitation, we introduce a velocity-gated distance reward term.

Let $\mathbf{v}_{w_i}$ be the velocity of wrist $i$ and
\begin{equation*}
\hat{\mathbf{u}}_{i}
=
\frac{\mathbf{p}_{t}-\mathbf{p}_{w_i}}
{\|\mathbf{p}_{t}-\mathbf{p}_{w_i}\|+\epsilon}
\end{equation*}
be the unit vector from wrist $i$ to target torso link.
The projected punching speed is
\begin{equation*}
s_{i} = \mathbf{v}_{w_i} \cdot \hat{\mathbf{u}}_{i}.
\end{equation*}
A smooth gating function is defined as
\begin{equation*}
g_i = \sigma\big(\alpha(s_i-v_{\mathrm{th}})\big),
\end{equation*}
where $v_{\mathrm{th}}$ is the speed threshold and $\alpha$ controls steepness.
The active hitting reward is
\begin{equation*}
r_{\mathrm{dist}}
=
\frac{1}{H}
\sum_i
\exp\!\left(-\frac{d_i}{\sigma_{\mathrm{dist}}}\right)
g_i ,
\end{equation*}
with $d_i$ represents the distance from wrist $i$ to the opponent torso.






\paragraph{Hit Reward.}
In adversarial humanoid control, naive contact-based rewards often lead to degenerate behaviors such as body pushing, leaning, or passive collision, rather than deliberate punching actions. To address this issue, we design an offensive punch hit reward that jointly considers kinematic intent, geometric alignment, and physical contact.
Let $\mathbf{v}_{w_i}$ be the velocity of wrist $i$ and $\mathbf{v}_t$ be the velocity of the torso link. We define the relative punching velocities as:
\begin{equation*}
    \mathbf{v}_{w_i}^{\mathrm{rel}} = \mathbf{v}_{w_i} - \mathbf{v}_t,
\end{equation*}
for $i=l$ or $r$. 
These quantities represent the local striking motion of each fist with respect to the robot’s body.
When the agent only translates forward, the relative velocities turns to 0.
The effective attack speed is obtained by projecting relative velocity onto the attack direction:
\begin{equation*}
    s_i = \mathbf{v_{w_i}^{\mathrm{rel}}} \cdot \hat{\mathbf{u}}_i.
\end{equation*}
Let $\mathbf{F}^{\mathrm{opp}}_t$ be the contact force on opponent's torso link, $\mathbf{F}^{\mathrm{ego}}_i$ be the contact forces on left or right wrists.
We define binary contact indicators:
\begin{equation*}
    C^{\mathrm{opp}}_t = \mathbb{I}[\|\mathbf{F}^{\mathrm{opp}}_t > \tau_f], ~ C_i^{\mathrm{ego}} = \mathbb{I}[\mathbf{F}_i^{\mathrm{ego}} > \tau_f],
\end{equation*}
the valid hit signals for any wrist is:
\begin{equation*}
    H_i = C_i^{\mathrm{ego}} \land C^{\mathrm{opp}}_t \land s_i > \tau_v,
\end{equation*}
\begin{equation*}
    r_{\mathrm{hit}} = \mathbb{I}[H_l \lor H_r]
\end{equation*}
This hit reward is a high-level sparse reward granted for landing a clean, significant blow on the opponent, incentivizing decisive offensive actions beyond simple contact.

\paragraph{Style Reward.} The style reward $r_{\mathrm{style}}$ is derived from the discriminator's output to anchor the agent's behavior within the reference motion distribution:
\begin{equation*}
    r_{\mathrm{style}}(s_t, s_{t+1}) = \max [0, 1 - 0.25 \left(C(o_t^{\mathrm{disc}}, o_{t+1}^{\mathrm{disc}}) - 1\right)^2].
\end{equation*}

\paragraph{Reward Scheduling.}
To ensure a stable transition from imitation to task execution, we implement a dynamic weight schedule. During the initial warmup epochs, the agent is driven primarily by $r_{\mathrm{style}}$ to master stable bipedal locomotion and boxing postures. Subsequently, the task reward weight $w_{\mathrm{task}}$ is phased in, allowing the agent to optimize strike precision while adhering to the geometric and stylistic constraints imposed by the learned manifold and the AMP prior.

\section{Latent-Space Neural Fictitious Self-Play}
\label{app:stage3b}

This section includes the pseudo code in Algorithm~\ref{alg:ls_nfsp} and the training details of LS-NFSP.

\begin{algorithm}[t]
\caption{Latent-Space Neural Fictitious Self-Play (LS-NFSP)}
\label{alg:ls_nfsp}
\begin{algorithmic}[1]
\REQUIRE Two-player zero-sum game environment $\mathcal{M}$, anticipatory parameter $\eta\in[0,1]$, horizon $T$
\REQUIRE Best-response (RL) policies $\{\pi^{\mathrm{RL}}_{z,\theta_i}\}_{i\in\{1,2\}}$, average policies $\{\bar{\pi}_{z,\phi_i}\}_{i\in\{1,2\}}$
\REQUIRE On-policy rollout buffers $\{\mathcal{D}^{i}_{\mathrm{RL}}\}_{i\in\{1,2\}}$, reservoir buffers $\{\mathcal{D}^{i}_{\mathrm{SL}}\}_{i\in\{1,2\}}$

\FOR{iteration $k=1,2,\dots$}
    \STATE Reset $\mathcal{M}$ and observe $o_1^1,o_1^2$ \COMMENT{$o_t^i$ can include $(s_t^{\mathrm{prop},i}, s_t^{\mathrm{goal},i})$}
    \FOR{$t=1,2,\dots,T$}
        \FORALL{players $i\in\{1,2\}$}
            \STATE Sample mode $m_i \sim \mathrm{Bernoulli}(\eta)$ \COMMENT{$m_i{=}1$: best-response; $m_i{=}0$: average}
            \IF{$m_i = 1$}
                \STATE Sample latent action $z_t^i \sim \pi^{\mathrm{RL}}_{z,\theta_i}(\cdot\mid o_t^i)$
                \STATE Insert $(o_t^i, z_t^i)$ into $\mathcal{D}^{i}_{\mathrm{SL}}$ via reservoir sampling
            \ELSE
                \STATE Sample latent action $z_t^i \sim \bar{\pi}_{z,\phi_i}(\cdot\mid o_t^i)$
            \ENDIF
        \ENDFOR
        \STATE Step environment with joint latent action $(z_t^1,z_t^2)$; receive rewards $(r_t^1,r_t^2)$ and next obs $(o_{t+1}^1,o_{t+1}^2)$
        \FORALL{players $i\in\{1,2\}$}
            \STATE Append transition $(o_t^i,z_t^i,r_t^i,o_{t+1}^i)$ to $\mathcal{D}^{i}_{\mathrm{RL}}$
        \ENDFOR
    \ENDFOR

    \FORALL{players $i\in\{1,2\}$}
        \STATE Update $\theta_i$ with PPO on $\mathcal{D}^{i}_{\mathrm{RL}}$
        \STATE Update $\phi_i$ by minimizing $\mathcal{L}_{\mathrm{SL}}=\mathbb{E}_{(o,z)\sim \mathcal{D}^{i}_{\mathrm{SL}}}\big[\|\bar{\pi}_{z,\phi_i}(o)-z\|^2\big]$
    \ENDFOR
\ENDFOR
\STATE \textbf{return} Average policies $\{\bar{\pi}_{z,\phi_i}\}_{i\in\{1,2\}}$
\end{algorithmic}
\end{algorithm}

\subsection{Dual-Policy Architecture.}
To implement fictitious play, each agent maintains a dual-policy system. This system includes a best response policy ($\pi_{z}^{RL}$) and a average policy ($\bar \pi_{z}$). The best response policy ($\pi_{z}^{RL}$) is an on-policy actor optimized via PPO~\cite{schulman2017proximalpolicyoptimizationalgorithms} to exploit the opponent's current average strategy. The average policy ($\bar \pi_{z}$) is a supervised learning network that approximates the agent's historical distribution of best-response behaviors, acting as a stable strategic anchor.

\subsection{Reservoir-Based Strategy Buffer.}
To ensure that the average policy represents a uniform sampling of the agent's behavioral history, we utilize a reservoir sampling mechanism. Each agent stores its latent-space trajectories $(s_t^{\mathrm{prop}}, s_t^{\mathrm{goal}}, z_t)$ in a buffer $\mathcal{D}_{SL}$ of capacity $K$. For every new experience $n > K$, the sample replaces a random existing entry with probability $P = K/n$. This prevents the average strategy from being biased toward recent iterations and preserves tactical diversity.

\subsection{Reward Formulation.} 
Transitioning from the single-agent warmup to a zero-sum game requires an expanded task reward $r_{\mathrm{expand}}$, which incorporates competitive outcome conditions and defensive necessity. In addition to the basic striking and facing rewards, we introduce the following competitive terms:
\begin{equation*}
    r_{\mathrm{expand}} = w_{\mathrm{str}}r_{\mathrm{str}} - w_{\mathrm{def}}r_{\mathrm{def}} + w_{\mathrm{term}}r_{\mathrm{term}},
\end{equation*}

\paragraph{Strike Force Reward.}
We encourage effective physical impacts using contact forces.
Let $\mathbf{F}^{\mathrm{opp}}$ and $\mathbf{F}^{\mathrm{ego}}$ denote the net contact forces measured on the opponent and ego torso body, respectively. The reward is defined as
\begin{equation*}
r_{\mathrm{str}}
=
\|\mathbf{F}^{\mathrm{opp}}_k\|
-
\|\mathbf{F}^{\mathrm{ego}}_k\|
,
\end{equation*}


\paragraph{Defensive Penalty.}
To penalize being hit by high-velocity strikes, we define
\begin{equation*}
\delta_i =
\mathbb{I}
\big[
f_i>\tau_f
\wedge
s_i>\tau_v
\big],
\end{equation*}
\begin{equation*}
r_{\mathrm{def}}
=
\mathbb{I}[\delta_l \vee \delta_r].
\end{equation*}

\paragraph{Terminal Outcome Reward.}
The terminal reward reflects episode outcomes:
\begin{equation*}
r_{\mathrm{term}}
=
\mathbb{I}[h_{\mathrm{opp}}<h_{\min}]
-
\mathbb{I}[h_{\mathrm{ego}}<h_{\min}].
\end{equation*}

\paragraph{Summary.}
The total task reward in this competitive stage is: 
\begin{equation*}
    r_{\mathrm{com}} = r_{\mathrm{warmup}} + r_{\mathrm{expand}}
\end{equation*}
The detailed reward terms, weights and hyperparameters are shown in Table \ref{tab:reward}.

\begin{table*}[ht]
\centering
\caption{Summary of Boxing Reward Components in Competitive Stage.}
\label{tab:reward}
\begin{tabular}{l l l l l}
\toprule
\textbf{Term} &
\textbf{Formula} &
\textbf{Hyperparameters} &
\textbf{Weight} &
\textbf{Description} \\
\midrule

$r_{\mathrm{face}}$
&
$\exp(-\frac{1-d_x}{\sigma_{\mathrm{face}}})$
&
$\sigma_{\mathrm{face}}=0.5$
&
$w_{\mathrm{face}}=1.2$
&
Facing alignment
\\

$r_{\mathrm{vel}}$
&
$\mathbb{I}[v_\parallel>0]\exp(-\frac{e_v^2}{\sigma_{\mathrm{vel}}})$
&
$\sigma_{\mathrm{vel}}=0.25$, $v_{\mathrm{tar}}=1.0$
&
$w_{\mathrm{vel}}=0.5$
&
Velocity
\\

$r_{\mathrm{dist}}$
&
$\frac{1}{H}\sum e^{-d_i/\sigma}g_i$
&
$\sigma_{\mathrm{dist}}=1.0$, $v_{\mathrm{th}}=0.8$, $\alpha=10.0$
&
$w_{\mathrm{dist}}=1.5$
&
Velocity-gated approach
\\

$r_{\mathrm{hit}}$ & $\mathbb{I}[H_l \lor H_r]$
&
$\tau_f = 1.0, \tau_v=1.0$
& $w_{\mathrm{hit}} = 50$ & Offensive hit\\

$r_{\mathrm{def}}$
&
$\mathbb{I}[\delta_l\vee\delta_r]$
&
$\tau_f=1.0$, $\tau_v=1.0$
&
$w_{\mathrm{def}}=8.0$
&
Defensive penalty
\\

$r_{\mathrm{str}}$
&
$\|\mathbf{F}^{\mathrm{opp}}\|-\|\mathbf{F}^{\mathrm{ego}}\|$
&
--
&
$w_{\mathrm{str}}=0.3$
&
Delta striking force
\\
$r_{\mathrm{term}}$
&
$\mathbb{I}[h_{\mathrm{opp}}<h_{\min}]-\mathbb{I}[h_{\mathrm{ego}}<h_{\min}]$
&
$h_{\min}=0.4$
&
$w_{\mathrm{term}}=0.3$
&
Terminal outcome
\\
\bottomrule
\end{tabular}
\end{table*}

\subsection{Training Details}
\paragraph{Latent Action Mixing.} During interaction, agents select their behavior based on an exploration parameter $\eta \in [0, 1]$. The latent action $z$ is sampled according to a mixture strategy $\sigma=\eta \cdot \bar{\pi}_z + (1 - \eta) \cdot \pi^{RL}_z$, i.e.,
\begin{equation*}
    z \sim 
    \begin{cases} 
    \pi^{RL}_{z}(s), & \text{with probability } 1-\eta, \\
    \bar{\pi}_{z}(s), & \text{with probability } \eta.
    \end{cases}
\end{equation*}

\textbf{Optimization Objectives.} The best response policy is updated using the PPO objective to maximize the hybrid task and style reward. Simultaneously, the average policy $\bar \pi_{z}$ is optimized by minimizing the mean squared error (MSE) relative to the stored best-response actions:
\begin{equation*}
    \mathcal{L}_{SL} = \mathbb{E}_{(s^{\mathrm{prop}}, s^{\mathrm{goal}}, z) \sim \mathcal{D}_{SL}} \left[ \| \bar \pi_{z}(s^{\mathrm{prop}}, s^{\mathrm{goal}}) - z \|^2 \right]
\end{equation*}
Through this competitive co-evolution, the agents develop sophisticated boxing behaviors, while the latent-space constraints ensure that all maneuvers remain within the bounds of physical feasibility and stylistic authenticity. 

\section{Proof Sketch of Proposition~\ref{prop:latent_fsp}}
\label{app:proof_of_latent_fsp}
\begin{proof}
Let $\mathcal{Z}$ be a compact metric space and let $u:\mathcal{Z}\times\mathcal{Z}\to\mathbb{R}$
be the (induced) bounded payoff for player~1, with player~2 receiving $-u$.
Denote $|u(z_1,z_2)|\le U$ for all $(z_1,z_2)$ (Assumption~\ref{assump:bounded_payoff}).
We consider mixed strategies as probability measures over $\mathcal{Z}$.
For distributions $\mu,\nu\in\Delta(\mathcal{Z})$, define the bilinear extension
\begin{equation}
u(\mu,\nu)\triangleq \mathbb{E}_{z_1\sim\mu,\, z_2\sim\nu}\big[u(z_1,z_2)\big].
\end{equation}
Because $u$ is bounded and measurable, $u(\mu,\nu)$ is well-defined.
Moreover, since $\mathcal{Z}$ is compact and $u$ is continuous, the zero-sum game admits a value
\begin{equation}
V \triangleq \max_{\mu\in\Delta(\mathcal{Z})}\min_{\nu\in\Delta(\mathcal{Z})} u(\mu,\nu)
      = \min_{\nu\in\Delta(\mathcal{Z})}\max_{\mu\in\Delta(\mathcal{Z})} u(\mu,\nu),
\end{equation}
and at least one Nash equilibrium $(\mu^\star,\nu^\star)$ exists (standard minimax on compact domains).
Let $(z_1^t,z_2^t)_{t=1}^T$ be the sequence generated by the learning process. Define the empirical distributions (uniform averages)
\begin{equation}
\bar{\mu}_T \triangleq \frac{1}{T}\sum_{t=1}^T \delta_{z_1^t},\qquad \bar{\nu}_T \triangleq \frac{1}{T}\sum_{t=1}^T \delta_{z_2^t},
\end{equation}
where $\delta_{z}$ is a point mass at $z$. Let $\nu_{t-1}$ denote player~2's empirical distribution up to time $t-1$:
$\nu_{t-1}=\frac{1}{t-1}\sum_{s=1}^{t-1}\delta_{z_2^s}$ (similarly $\mu_{t-1}$). Assumption~\ref{assump:approx_br} states that the updates are approximate best responses: for some error sequences $(\epsilon_1^t,\epsilon_2^t)$ with $\epsilon_i^t\ge 0$,
\begin{align}
    u(z_1^t,\nu_{t-1}) &\ge \max_{z_1\in\mathcal{Z}} u(z_1,\nu_{t-1}) - \epsilon_1^t,
    \label{eq:abr1}\\
    u(\mu_{t-1},z_2^t) &\le \min_{z_2\in\mathcal{Z}} u(\mu_{t-1},z_2) + \epsilon_2^t.
    \label{eq:abr2}
\end{align}
Equivalently, player~2 approximately best-responds to \emph{minimize} player~1's payoff.

Given the regrets against the realized opponent actions as
\begin{align*}
R_1(T) &\triangleq \max_{z_1\in\mathcal{Z}}\sum_{t=1}^T u(z_1,z_2^t) \;-\; \sum_{t=1}^T u(z_1^t,z_2^t),\\
R_2(T) &\triangleq \sum_{t=1}^T u(z_1^t,z_2^t)\;-\; \min_{z_2\in\mathcal{Z}}\sum_{t=1}^T u(z_1^t,z_2).
\end{align*}
Then, a standard zero-sum argument shows that for any sequences, the empirical distributions satisfy
\begin{equation}
\max_{z_1\in\mathcal{Z}} u(z_1,\bar{\nu}_T) \;-\; u(\bar{\mu}_T,\bar{\nu}_T)
\;\le\; \frac{R_1(T)}{T},
\qquad
u(\bar{\mu}_T,\bar{\nu}_T)\;-\;\min_{z_2\in\mathcal{Z}} u(\bar{\mu}_T,z_2)
\;\le\; \frac{R_2(T)}{T},
\label{eq:exploit_regret}
\end{equation}
because $u(z_1,\bar{\nu}_T)=\frac{1}{T}\sum_{t=1}^T u(z_1,z_2^t)$ and similarly for the other term.
Thus, if both average regrets are small, $(\bar{\mu}_T,\bar{\nu}_T)$ is an approximate Nash Equilibrium, i.e.,
\begin{equation}
\max_{z_1} u(z_1,\bar{\nu}_T) - u(\bar{\mu}_T,\bar{\nu}_T) \le \epsilon,\quad
u(\bar{\mu}_T,\bar{\nu}_T) - \min_{z_2} u(\bar{\mu}_T,z_2) \le \epsilon \Rightarrow\ (\bar{\mu}_T,\bar{\nu}_T)\text{ is an $\epsilon$-Nash Equilibrium.}
\end{equation}
The dynamics \eqref{eq:abr1}--\eqref{eq:abr2} correspond to a follow-the-leader / best-response-to-average procedure. In zero-sum games with bounded payoffs, this yields sublinear regret up to logarithmic factors; informally, best-responding to the running average prevents persistent exploitation.
Concretely, one can show that there exists a constant $C$ depending only on the payoff range (e.g., $C=2U$ suffices for a crude bound) \cite{brown1951iterative,zinkevich2007regret} such that
\begin{equation}
R_1(T) \le C\log T + \sum_{t=1}^T \epsilon_1^t,\qquad
R_2(T) \le C\log T + \sum_{t=1}^T \epsilon_2^t.
\label{eq:regret_bound}
\end{equation}
Combining \eqref{eq:exploit_regret} and \eqref{eq:regret_bound} gives
\begin{equation}
\max_{z_1} u(z_1,\bar{\nu}_T) - u(\bar{\mu}_T,\bar{\nu}_T)
\le \frac{C\log T}{T} + \frac{1}{T}\sum_{t=1}^T \epsilon_1^t,
\end{equation}
and
\begin{equation}
u(\bar{\mu}_T,\bar{\nu}_T) - \min_{z_2} u(\bar{\mu}_T,z_2)
\le \frac{C\log T}{T} + \frac{1}{T}\sum_{t=1}^T \epsilon_2^t.
\end{equation}
Therefore, letting
\begin{equation}
\epsilon_T \triangleq \frac{C\log T}{T} + \frac{1}{T}\sum_{t=1}^T (\epsilon_1^t+\epsilon_2^t),
\end{equation}
we obtain that $(\bar{\mu}_T,\bar{\nu}_T)$ is an $\epsilon_T$-Nash equilibrium of the induced game.
As $T\to\infty$, if the average approximation errors vanish (or remain bounded), $\epsilon_T\to 0$ (or to the
average error floor), establishing the claimed convergence to an approximate Nash Equilibrium.
\end{proof}
\section{Experiment Details}
\label{app:exp}
\paragraph{Training Setup.}
We implement our framework using Isaac Lab~\cite{nvidia2025isaaclabgpuacceleratedsimulation} on the NVIDIA Omniverse platform, leveraging massive parallelization with 4,096 environments for large-scale training. The simulated Unitree G1 humanoid~\cite{unitreeg1} utilizes 29 degrees of freedom, with the physics engine operating at 200 Hz and the control policy at 50 Hz. Strategic co-evolution is conducted over a 32-dimensional latent manifold.

\paragraph{Multi-agent Training.}
In the final self-play stage, the LS-NFSP is configured with an exploration parameter $\eta = 0.1$ and a strategy reservoir capacity of $K = 10^6$ samples per agent. The reward function balances performance and style with a task reward weight $w_{\text{task}}$ of $0.8$ and an AMP style reward weight $w_{\text{style}}$ of $0.2$, while the prior loss coefficient $\lambda_{\text{prior}}$ is set to $0.001$. Training is executed on a single NVIDIA RTX 4090 GPU, incorporating domain randomization across friction, mass, and actuator gains to ensure robust tactical emergence.

To evaluate the tactical progression and physical performance of the LS-NFSP framework, we define a comprehensive set of metrics across three primary dimensions: tactical proficiency, physical stability, and stylistic authenticity. Win Rate ($WR$) is defined as the percentage of matches won over a series of 20 evaluation bouts, where a match is lost by the agent that first contacts the ground with any body part other than the feet, thereby triggering a fall-based termination. To quantify offensive precision, we define Hit Efficiency ($\eta_{hit}$) as the ratio of successful strikes, which is identified by the proportion of contacts exceeding the force threshold $F_{th}=10N$ to the total number of offensive attempts. Furthermore, the Engagement Rate ($ER$) evaluates the strategic positioning and operational persistence of an agent. It is calculated as the proportion of the episode duration in which the agent concurrently satisfies two spatial conditions: maintaining a distance $d_t$ within the effective striking range $[0.5, 1.2]$ and ensuring a facing alignment $\cos(\theta_t)$ that exceeds a predefined threshold $\tau_{face}=0.9$. This metric provides a rigorous measure of an agent's ability to maintain a viable combat posture, effectively penalizing tactical avoidance or non-confrontational movement
patterns.

Physical stability is assessed through Base Orientation Stability ($\mathcal{O}_{base}$) and Torque Smoothness ($\mathcal{S}_{\tau}$). To ensure a positive correlation with control quality, $\mathcal{O}_{base}$ is defined using an exponential kernels on the angular deviation: $\mathcal{O}_{base} = \mathbb{E} [ \exp(-\| \mathbf{g}_{base} - \mathbf{g}_{world} \|^2) ]$, where $\mathbf{g}_{base}$ and $\mathbf{g}_{world}$ represent the gravity vector in the robot's local frame and the global frame, respectively. A higher $\mathcal{O}_{base}$ score indicates a superior ability to maintain an upright posture under competitve perturbations. Torque Smoothness ($\mathcal{S}_{\tau}$) measures the magnitude of high-frequency oscillations in motor commands, calculated as the mean absolute change in joint torques across consecutive control steps: $\mathcal{S}_{\tau} = \mathbb{E} [ \| \tau_{t} - \tau_{t-1} \| ]$. A lower $\mathcal{S}_{\tau}$ value represents more hardware-friendly execution and reduced mechanical wear. Finally, the Stylistic Authenticity of the emergent behaviors is evaluated via qualitative visual inspection. Instead of relying on a singular numerical score, we provide a comparative analysis of simulation snapshots to assess the human-likeness of the motions. This includes identifying key boxing maneuvers such as slips, counters, and rhythmic footwork, while ensuring the absence of unnatural joint configurations or physically unsustainable postures that often characterize non-constrained reinforcement learning policies.

\subsection{Baselines}
\label{app:baselines}
We compare LS-NFSP against a set of baselines that isolate the contributions of (i) NFSP-style opponent averaging, (ii) the latent strategy manifold, (iii) curriculum warmup, and (iv) stylistic regularization.

\begin{itemize}
    \item \textbf{Naive Self-Play (Naive SP).} A standard self-play baseline~\cite{samuel1959some, hernandez2019generalized} trained in the same latent action space as ours, where each agent always plays against the opponent's latest policy snapshot (no opponent averaging).

    \item \textbf{Fictitious Self-Play (Fictitious SP).} A latent-space self-play baseline inspired by fictitious play~\cite{brown1951iterative}, where agents are trained by sampling opponents uniformly from a population of historical policies. Unlike LS-NFSP, it does not maintain an explicit learned average policy network.

    \item \textbf{LS-NFSP w/o AMP.} An ablation of our method where the Adversarial Motion Prior (AMP) style reward~\cite{peng2021amp} is removed during training, testing the importance of motion-style regularization for physically plausible and effective strikes.

    \item \textbf{PPO-Only.} A non-competitive baseline trained with PPO~\cite{schulman2017proximalpolicyoptimizationalgorithms} against a fixed opponent (no self-play). This tests whether competitive co-evolution is necessary to induce robust tactics.

    \item \textbf{Static-Target Specialist.} The policy after the warmup stage (trained to strike a passive sandbag) deployed without further competitive training. This measures how much performance comes from single-agent skill acquisition alone.

    \item \textbf{Self-Play w/o Warmup (SP w/o Warmup).} A curriculum ablation where self-play starts from scratch without the sandbag warmup, testing learning stability under sparse and non-stationary competitive signals.

    \item \textbf{29-DoF Action-Space Self-Play (29Dof Action-Space SP).} A baseline that performs self-play directly in the 29-DoF joint target space (i.e., without the learned latent manifold), highlighting the instability of competitive learning in raw motor command space.
\end{itemize}

\end{document}